\PassOptionsToPackage{table}{xcolor}
\documentclass{article}

\usepackage{microtype}
\usepackage{graphicx}
\usepackage{subcaption}
\usepackage{booktabs} 

\usepackage{hyperref}

\usepackage[preprint]{icml2026}

\usepackage{amsmath}
\usepackage{amssymb}
\usepackage{mathtools}
\usepackage{amsthm}

\usepackage[capitalize,noabbrev]{cleveref}

\usepackage{multirow}
\usepackage{array}

\usepackage[table]{xcolor}
\usepackage{booktabs}
\definecolor{bg-green}{RGB}{209, 239, 213}
\definecolor{text-red}{RGB}{255, 0, 0}

\theoremstyle{plain}
\newtheorem{theorem}{Theorem}[section]

\newtheorem{lemma}[theorem]{Lemma}

\theoremstyle{definition}
\newtheorem{definition}[theorem]{Definition}
\newtheorem{assumption}[theorem]{Assumption}
\theoremstyle{remark}

\usepackage[textsize=tiny]{todonotes}

\icmltitlerunning{Adaptive Robust Estimator for Multi-Agent Reinforcement Learning}

\usepackage{graphicx}   
\usepackage{multirow}   
\usepackage{xcolor}     
\definecolor{bgblue}{RGB}{220,235,255} 
\usepackage{float}
\usepackage{placeins}
\usepackage{algorithm}
\usepackage{algorithmic}

\newcommand{\E}{\mathbb{E}}
\newcommand{\Pbb}{\mathbb{P}}
\newcommand{\Var}{\mathrm{Var}}
\newcommand{\R}{\mathbb{R}}
\newcommand{\abs}[1]{\left|#1\right|}
\newcommand{\1}{\mathbf{1}}

\usepackage{ragged2e}
\usepackage{threeparttable}
\usepackage{enumitem}


\begin{document}

\twocolumn[
\icmltitle{Adaptive Robust Estimator for Multi-Agent Reinforcement Learning}




\begin{icmlauthorlist}
\icmlauthor{Zhongyi Li\textsuperscript{*}}{buaa}
\icmlauthor{Wan Tian\textsuperscript{*}}{pku}
\icmlauthor{Jingyu Chen}{cas}
\icmlauthor{Kangyao Huang}{thu}
\icmlauthor{Huiming Zhang}{buaa}
\icmlauthor{Hui Yang}{pku}
\icmlauthor{Tao Ren}{pku}
\icmlauthor{Jinyang Jiang}{pku}
\icmlauthor{Yijie Peng\textsuperscript{$\dagger$}}{pku}
\icmlauthor{Yikun Ban\textsuperscript{$\dagger$}}{buaa}
\icmlauthor{Fuzhen Zhuang\textsuperscript{$\dagger$}}{buaa}
\end{icmlauthorlist}

\icmlaffiliation{buaa}{Beihang University}
\icmlaffiliation{pku}{Peking University}
\icmlaffiliation{cas}{Chinese Academy of Sciences}
\icmlaffiliation{thu}{Tsinghua University}

\icmlcorrespondingauthor{Yijie Peng}{}
\icmlcorrespondingauthor{Yikun Ban}{}
\icmlcorrespondingauthor{Fuzhen Zhuang}{}

\icmlkeywords{Machine Learning, ICML}

\vskip 0.3in
]



\printAffiliationsAndNotice{
\textsuperscript{*}Equal contribution \\
\textsuperscript{$\dagger$}Corresponding authors}


\definecolor{bgblue}{RGB}{220, 230, 245}

\begin{abstract}
Multi-agent collaboration has emerged as a powerful paradigm for enhancing the reasoning capabilities of large language models, yet it suffers from interaction-level ambiguity that blurs generation, critique, and revision, making credit assignment across agents difficult. Moreover, policy optimization in this setting is vulnerable to heavy-tailed and noisy rewards, which can bias advantage estimation and trigger unstable or even divergent training. To address both issues, we propose a robust multi-agent reinforcement learning framework for collaborative reasoning, consisting of two components: Dual-Agent Answer–Critique–Rewrite (DACR) and an Adaptive Robust Estimator (ARE). DACR decomposes reasoning into a structured three-stage pipeline—answer, critique, and rewrite—while enabling explicit attribution of each agent’s marginal contribution to its partner’s performance. ARE provides robust estimation of batch experience means during multi-agent policy optimization. Across mathematical reasoning and embodied intelligence benchmarks, even under noisy rewards, our method consistently outperforms the baseline in both homogeneous and heterogeneous settings. These results indicate stronger robustness to reward noise and more stable training dynamics, effectively preventing optimization failures caused by noisy reward signals. Our code  is available at \href{https://github.com/bhai114/ARE}{https://github.com/bhai114/ARE}.
\end{abstract}

\section{Introduction} \label{secintro}

\begin{figure*}[!t]
\centering
\includegraphics[width=0.99\textwidth]{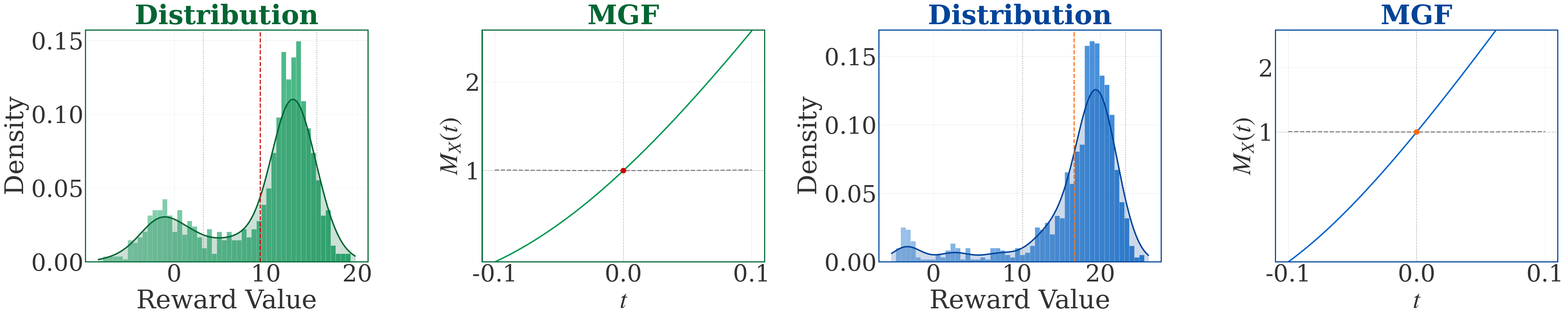}
\caption{The scoring results of the reward model. We randomly selected questions and answers from the MATH500 and Gaokao2023 datasets, and then used the large model to generate the answers. The first and third figures show the scoring results of the reward model for the answers generated by the large model. The second and fourth figures are the corresponding moment generating functions. It can be seen that the data distribution of the scoring by the reward model is biased.}
\label{fig:framework}
\vspace{-2mm} 
\end{figure*}

Large language models (LLMs) have achieved impressive performance on mathematical and logical reasoning tasks \cite{zhang2026heterogeneous,yang2026grouprelativeadvantagebiased,liu2025logical,wang2025survey,liu2025mathematical,huang2026does}. In this line of work, reinforcement learning (RL) has emerged as an increasingly common approach, providing a compelling alternative to supervised imitation learning \cite{ren2025riskpo,shao2025deepseekmath,zou2025transformer}. In parallel, multi-agent collaboration has emerged as a practical way to mitigate single-agent limitations such as confirmation bias and weak self-critique: pairing complementary roles (e.g., solver vs. evaluator) or aggregating multiple attempts can surface hidden errors and improve reliability \cite{wan2503rema,lin2025interactive,liao2025marft,chen2025harnessing}. These trends naturally motivate training collaborative LLM systems via multi-agent reinforcement learning (MARL). However, in practice, policy optimization for multi-agent reasoning is often brittle, and stability issues can dominate the training dynamics.

A major source of brittleness is the \emph{reward signal}. Reasoning rewards are frequently produced by verifiers, reward models, or heuristic graders, and can be noisy, biased, and heavy-tailed, with occasional extreme outliers (Figure~\ref{fig:framework}). Standard GRPO (Group Relative Policy Optimization) \cite{shao2024deepseekmathpushinglimitsmathematical} normalizes rewards using the batch empirical mean and variance. Under heavy-tailed or contaminated rewards, a small number of anomalous samples can significantly skew batch statistics, distort advantage scaling, and trigger oscillatory updates or even training divergence \cite{yang2026grouprelativeadvantagebiased}. This issue becomes more pronounced in multi-agent settings, where interaction can introduce additional variability and non-stationarity into the return distribution \citep{wang2025multi,huang2026real}.

To address these two challenges, we propose a robust MARL framework that stabilizes policy optimization under noisy, heavy-tailed rewards. The framework consists of two components. First, DACR is an interaction protocol that decomposes collaboration into explicit stages: each agent produces an initial answer, critiques its partner’s answer, and then rewrites its own solution conditioned on the critique. This structure decouples generation from evaluation and yields richer training trajectories than unstructured dialogues. We further adopt a lightweight cross-stage reward design to encourage effective interaction, without requiring complex credit assignment. Second, ARE is an adaptive robust estimator for MARL policy optimization. It replaces fragile batch-mean normalization of experience returns with a robust, adaptive location estimator, thereby improving training stability in the presence of outliers.

We evaluate the effectiveness of our framework on mathematical reasoning benchmarks, using both homogeneous and heterogeneous agent pairs across model families and scales. We further demonstrate the generality of ARE beyond language reasoning by applying it to an aerial vision-and-language navigation (VLN) task. Across both tasks, our approach consistently improves performance over GRPO-based baselines. Overall, our contributions are:

\begin{itemize}[leftmargin=*, topsep=2pt, itemsep=2pt, parsep=0pt, partopsep=0pt]
\item We introduce DACR, a structured collaboration protocol with a cross-improvement reward that attributes each agent’s contribution to its collaborator’s reasoning quality, enabling more interpretable and learnable credit assignment.

\item We propose ARE, which replaces GRPO’s fragile batch-mean advantage normalization with a robust estimator, mitigating instability in advantage estimation under noisy and heavy-tailed rewards.

\item We systematically evaluate both homogeneous and heterogeneous multi-agent systems on large-scale mathematical reasoning benchmarks, demonstrating consistent improvements over single-agent and multi-agent baselines in accuracy, stability, and generalization.

\item We further validate ARE in an embodied aerial VLN setting with a Qwen2-VL-2B VLM, achieving higher training rewards and improved NE/SR/OSR/SPL across all splits, with the largest gains on unseen environments.
\end{itemize}

\section{Methodology}
\label{sec:method}
In this section, we present our robust MARL framework, which consists of two components: DACR and ARE. DACR centers on a structured interaction protocol that facilitates collaborative error detection, while ARE provides a robust optimization procedure that mitigates the heavy-tailed rewards inherent to multi-agent settings.
\begin{figure*}[!t]
\centering
\includegraphics[width=0.99\textwidth]{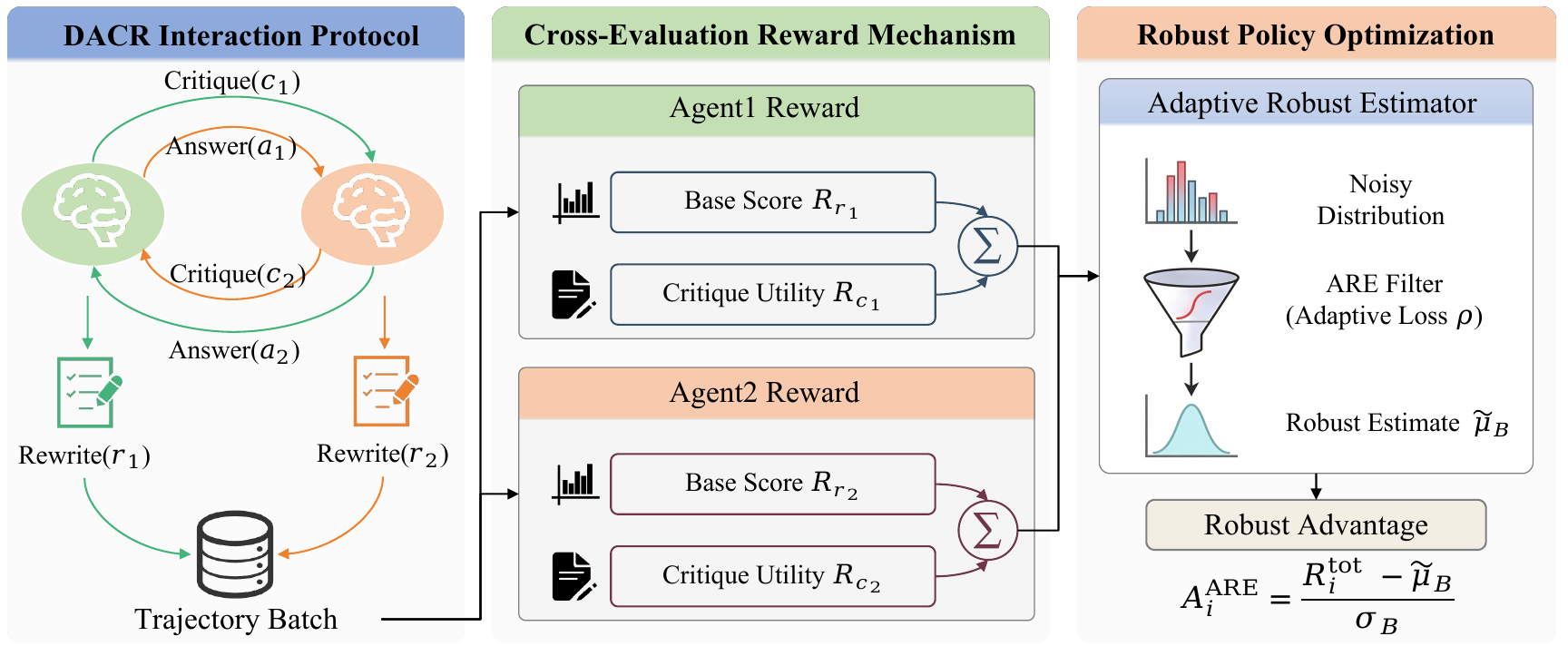}
\caption{The proposed DACR interaction framework. Two agents, $\pi_1$ and $\pi_2$, act as both solvers and peer-evaluators. The reward mechanism attributes credit based on standalone accuracy and the marginal improvement induced in the partner's revised solution.}
\label{fig:framework_new}
\vspace{-2mm}
\end{figure*}

\subsection{Dual-Agent Reasoning Interaction Protocol}
In this section, we introduce DACR, whose workflow is illustrated in the left panel of Figure~\ref{fig:framework_new}. DACR decouples the reasoning process into a dual-agent system in which two policies, $\pi_1^{\theta}$ and $\pi_2^{\phi}$, alternate between generative and evaluative roles. Given a problem $q$, the interaction proceeds through a structured three-stage trajectory:

\begin{enumerate}[leftmargin=*, topsep=2pt, itemsep=2pt, parsep=0pt, partopsep=0pt]
\item Independent Answering: Both agents generate initial reasoning chains and answers: $a_1 \sim \pi_1(q)$ and $a_2 \sim \pi_2(q)$.
\item Mutual Critique: Agents swap their answers. Agent 1 generates a critique identifying potential errors in $a_2$ ($c_1 \sim \pi_1(q, a_2)$), while Agent 2 critiques $a_1$ ($c_2 \sim \pi_2(q, a_1)$).
\item Revision: Each agent produces a refined answer $r_i$ by conditioning on the peer feedback received: $r_1 \sim \pi_1(q, a_1, d_2)$ and $r_2 \sim \pi_2(q, a_2, c_1)$.
\end{enumerate}

This protocol ensures that evaluation signals are external to the original reasoning trace, providing a more objective and adversarial diagnostic environment for solving complex problems.

\subsection{Cross-Evaluation Reward Mechanism}
In this section, we describe the multi-agent reward design for DACR, as illustrated in the middle panel of Figure~\ref{fig:framework_new}.

To optimize DACR, we move beyond binary correctness and adopt a lightweight reward mechanism that supports fine-grained credit attribution. Let $\mathrm{S}(y,q)$ denote a scoring function that assesses the logical and numerical validity of a response $y$ to problem $q$.

For each agent $i\in\{1,2\}$, we assign stage-wise rewards for the answering and rewriting phases:
\[
R_{\mathrm{ans},i}=\mathrm{S}(a_i,q),\qquad R_{\mathrm{rw},i}=\mathrm{S}(r_i,q).
\]
Crucially, we introduce a \textbf{cross-improvement} term $\Delta_i$ that credits an agent when its critique leads to a measurable improvement in its partner’s rewritten answer:
\[
\Delta_1 = R_{\mathrm{rw},2}-R_{\mathrm{ans},2},\qquad \Delta_2 = R_{\mathrm{rw},1}-R_{\mathrm{ans},1}.
\]
The total return for agent $i$ is then
\[
R_i^{\mathrm{total}} = R_{\mathrm{rw},i} + \gamma\,\Delta_i,
\]
where $\gamma>0$ trades off individual solution quality against collaborative utility. This objective encourages agents to provide actionable critiques while maintaining high standards for their own final answers.

\subsection{Robust Multi-Agent Policy Optimization}
In this section, we instantiate ARE (introduced in Section~\ref{AREintro}) for robust multi-agent policy optimization. The key idea is to stabilize advantage estimation by replacing the fragile batch mean with our proposed ARE-based location estimator. Given a batch of returns $\{R_k\}_{k=1}^B$, we compute a robust location estimate $\tilde{\mu}_B$ using ARE. The resulting robust advantage is defined as
\begin{equation*}
A_k^{\text{ARE}} = \frac{R_k^{\text{total}} - \tilde{\mu}_B}{\sigma_B + \epsilon}.
\end{equation*}
Each agent is then optimized by maximizing the clipped surrogate objective
\begin{equation*}
\begin{aligned}
\mathcal{L}(\theta) &= \mathbb{E}\Big[ \min\big( \eta\, A^{\text{ARE}},\ \text{clip}(\eta, 1-\epsilon, 1+\epsilon)\, A^{\text{ARE}} \big) \Big] \\
&\quad - \beta\, \mathbb{D}_{\text{KL}}(\pi_{\theta} \,\|\, \pi_{\text{ref}}),
\end{aligned}
\end{equation*}
where $\eta$ denotes the importance sampling ratio. By anchoring advantage estimation at a robust center, robust policy optimization suppresses gradient spikes induced by outlier rewards, yielding more stable convergence in challenging multi-agent reasoning landscapes.

\section{Adaptive Robust Estimator} \label{AREintro}


In this section, we introduce our proposed ARE method, which can be viewed as a principled refinement of the classical Median-of-Means (MoM) estimator \citep{sun2021we}. MoM partitions samples into blocks, computes a mean within each block, and aggregates the blockwise estimates via a median, providing robustness to heavy tails and a minority of corrupted blocks (see Appendix~\ref{MOMDiscussion} for details). However, MoM enforces robustness only at the aggregation stage: each blockwise estimate is still a sample mean and thus remains sensitive to within-block outliers. ARE addresses this weakness by replacing the blockwise mean with a robust estimator obtained by minimizing an adaptive loss \citep{BarronCVPR2019}, thereby enforcing robustness at both intra-block and inter-block levels and improving stability under heavy-tailed contamination. We defer the adaptive loss definition and its optimization properties to Appendix~\ref{IntroductionAdaptive}.


Specifically, following the MoM template, we partition the samples $X_1,\dots,X_n$ into $k$ disjoint blocks $B_1,\dots,B_k$. For each block, we compute a robust location estimate by minimizing the adaptive loss:
\begin{equation}\label{adaptive_loss_es}
\widetilde{X}_j \in \arg\min_{x\in\mathbb{R}}
\frac{1}{|B_j|}\sum_{i\in B_j}\rho\!\big(X_i-x;\alpha,c\big), \quad j=1,\dots,k .
\end{equation}
We then define ARE as the median of these blockwise estimates:
\begin{equation}\label{ARE}
\widetilde{\mu}=\operatorname{median}(\widetilde{X}_1,\dots,\widetilde{X}_k).
\end{equation}

The main computational challenge is solving \eqref{adaptive_loss_es}, since $\rho(\cdot;\alpha,c)$ can be nonconvex and involves robustness parameters $(\alpha,c)$. Without loss of generality, consider a single block of size $m$:
\begin{equation}\label{generalform}
\widetilde{X}\in\arg\min_{x\in\mathbb{R}}\frac{1}{m}\sum_{i=1}^m \rho\!\big(X_i-x;\alpha,c\big).
\end{equation}
Tuning $(\alpha,c)$ via cross-validation or Lepski's method is computationally prohibitive, especially in multi-agent settings. We therefore use alternating optimization: initialize $\widetilde{X}^{(0)}$ and $(\alpha^{(0)},c^{(0)})$ (e.g., $\widetilde{X}^{(0)}$ as the sample median, $\alpha^{(0)}=1$, $c^{(0)}=1$), and iterate
\begin{align}
\alpha^{(t+1)}, c^{(t+1)} &= \arg\min_{\alpha\in(-\infty,2],\,c>0}\frac{1}{m}\sum_{i=1}^m \rho\!\big(X_i-\widetilde{X}^{(t)};\alpha,c\big),\label{parastep}\\
\widetilde{X}^{(t+1)} &= \arg\min_{x\in\mathbb{R}}\frac{1}{m}\sum_{i=1}^m \rho\!\big(X_i-x;\alpha^{(t+1)},c^{(t+1)}\big),\label{Xstep}
\end{align}
until convergence. Section~\ref{solution} details efficient solvers for \eqref{parastep} and \eqref{Xstep}.

\section{Optimization Perspective} \label{solution}
In this section, we outline tailored solution procedures for the nonconvex subproblems \eqref{parastep} and \eqref{Xstep}, respectively.

\subsection{Estimation of Robust and Scale Parameters}

Prior to estimating the adaptive loss parameters, we first derive the following key insight by differentiating the adaptive loss with respect to \(\alpha\).

\begin{lemma} \label{Monotonicitybaron}
Fix $\epsilon\in\mathbb{R}$ and $c>0$. For $\alpha\neq 0,2$. Then
\(
\frac{\partial}{\partial\alpha}\rho(\epsilon;\alpha,c)\ \ge\ 0,
\)
and the inequality is strict whenever $X\neq 0$.
\end{lemma}

\begin{figure}[H]
\centering
\includegraphics[scale=0.23]{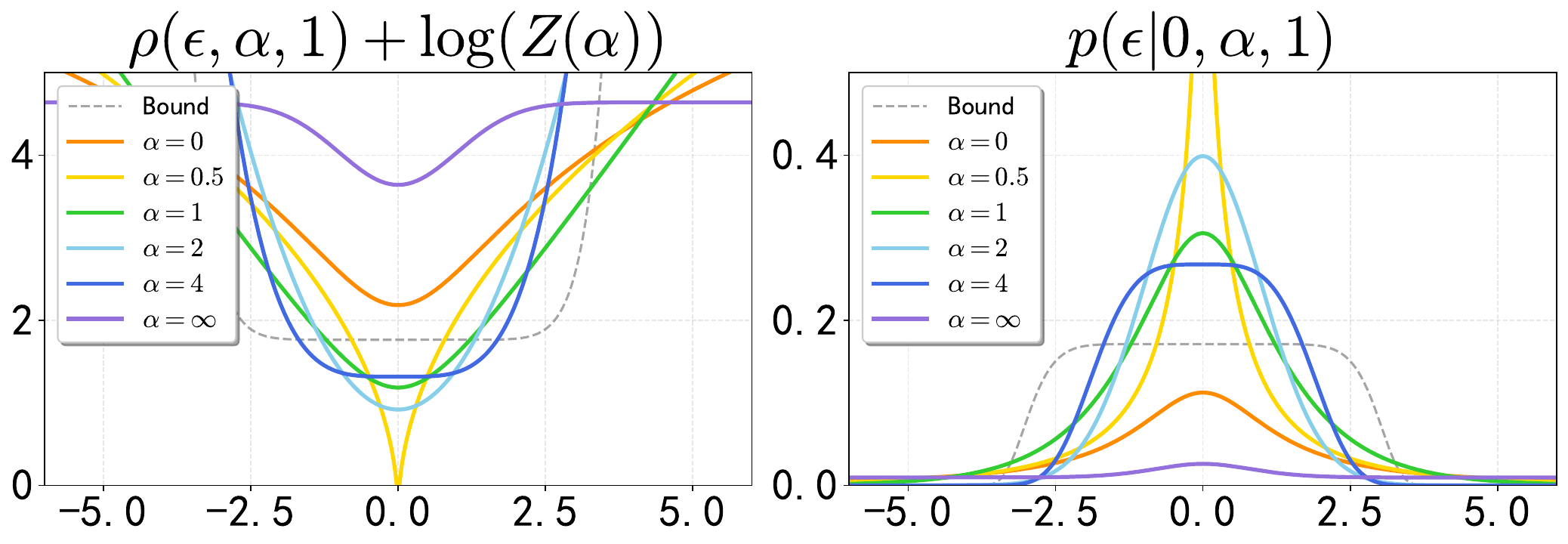}
\caption{The probability distribution and its negative log-likelihood corresponding to the adaptive loss function. The left panel shows the negative log-likelihood function $-\log p(\epsilon;\alpha,1) = \rho(\epsilon;\alpha,1) + \log Z(\alpha)$ for different shape parameters $\alpha = 0, 0.5, 1, 2, 4, \infty$; The right panel shows the corresponding probability density function $p(\epsilon;\alpha,1) = \exp\{-\rho(\epsilon;\alpha,1)\}/Z(\alpha)$.}\label{adaptivellospdfnll} 
\end{figure}

Lemma~\ref{Monotonicitybaron} indicates that, although it may seem natural to estimate the parameters by directly solving~\eqref{parastep}, this approach is in fact ill-posed: the objective admits a degenerate solution that pushes the shape parameter \(\alpha\) to excessively small values.
To avoid this collapse, we follow \citet{BarronCVPR2019,jung2024adaptive} and adopt a likelihood-based calibration.

Concretely, we consider the exponential-family density
\(
p(\epsilon;\alpha,c) =\exp(-\rho(\epsilon;\alpha,c)) / cZ(\alpha)
\)
with \(Z(\alpha)=\int_{-\infty}^{\infty}\exp(-\rho(u;\alpha,1))du\) so that \(Z(\alpha)\) normalizes the distribution. The corresponding negative log-likelihood is
\(
-\log p(\epsilon;\alpha,c)
\;=\;
\rho(\epsilon;\alpha,c) + \log c + \log Z(\alpha).
\)

Figure~\ref{adaptivellospdfnll} shows the induced density and its negative log-likelihood, clarifying the effect of $\alpha$. Decreasing $\alpha$ flattens the tails, reducing the penalty on large residuals (outliers) while increasing the relative penalty near the origin (inliers). Estimation therefore involves a trade-off: smaller $\alpha$ discounts large errors but penalizes small ones more, whereas larger $\alpha$ approaches least squares, favoring inliers while charging more for outliers. This trade-off encourages the method to adapt its robustness to the empirical residual distribution, avoiding the degenerate behavior implied by Lemma~\ref{Monotonicitybaron}.

Following the aforementioned discussion, we estimate the parameters $\alpha$ and $c$ by minimizing the negative log-likelihood function based on the sample, which is defined as
\begin{equation}\label{realstep1}
\begin{aligned}
\alpha^{(1)},\, c^{(1)}
&= \arg\min_{\alpha,\,c>0}
\sum_{i=1}^m \rho\bigl(X_i - \widetilde{X}^{(0)};\alpha,c\bigr)
\\
& + m\bigl(\log c + \log Z(\alpha)\bigr).
\end{aligned}
\end{equation}

It is worth noting that in optimization problem (\ref{realstep1}), we need to jointly optimize \(\alpha\) and \(c\), which differs from methods such as \citet{jung2024adaptive, 9361339} that optimize only \(\alpha\), and may entail greater computational challenges.

\subsection{Estimation of the Mean} \label{Estimationmean}
Upon obtaining $\alpha^{(1)}$ and $c^{(1)}$, directly solving \eqref{Xstep} remains nontrivial since the adaptive robust loss is generally nonconvex.
We therefore adopt the graduated nonconvexity (GNC) paradigm~\citep{10.1007/BF00131148} within an iteratively reweighted least-squares (IRLS) framework~\citep{9361339}.

Let $\epsilon_i = X_i - X$.
By matching the gradients of \eqref{Xstep} and the weighted least-squares objective, \eqref{Xstep} is equivalent to
\begin{equation}\label{weightedls}
\widetilde{X}^{(1)} \;=\; \arg\min_{X\in\mathbb{R}} \frac{1}{2m}\sum_{i=1}^{m} w_i \,(X_i - X)^2,
\end{equation}
where the weights are induced by the loss through
\(
w_i \;=\; \frac{1}{\epsilon_i}\frac{\partial \rho(\epsilon_i;\alpha^{(1)},c^{(1)})}{\partial \epsilon_i}.
\)
However, if $w_i$ are treated as unconstrained free variables and optimized jointly with $X$ in \eqref{weightedls}, the objective admits a degenerate minimum at $w_i=0$ for all $i$ (independent of $X$), rendering the solution meaningless.

To prevent this degeneration, we introduce an outlier-process regularizer on the weights via the Black--Rangarajan duality~\citep{10.1007/BF00131148}.
Define $\phi(z)\coloneqq \rho(c\sqrt{z};\alpha,c)$.
For the adaptive robust loss family, $\phi'(z)$ satisfies
\(
\lim_{z\to 0}\phi'(z)=\frac{1}{2}, \lim_{z\to\infty}\phi'(z)=0, \phi''(z)<0,
\)
which guarantees the existence of an analytical outlier-process function $\Phi_\rho(\cdot)$ such that
\[
\rho(\epsilon;\alpha,c)\;=\;\min_{w\in(0,1]}\frac{1}{2}w\epsilon^2+\Phi_\rho(w).
\]
Consequently, \eqref{weightedls} can be rewritten as the following constrained and regularized problem:
\begin{equation}\label{dualityweight}
\widetilde{X}^{(1)}
\;=\;
\arg\min_{X\in\mathbb{R},\, w_i\in(0,1]} \frac{1}{2m}\sum_{i=1}^m
\Big( w_i (X_i - X)^2 + \Phi_\rho(w_i)\Big).
\end{equation}

When initialization is poor, directly optimizing \eqref{dualityweight} remains challenging due to nonconvexity.
We therefore solve it by GNC: we start from a convex quadratic surrogate and gradually introduce nonconvexity until the original adaptive loss is recovered.
Concretely, we employ a shape mapping $f(\beta,\alpha^{(1)})$ that interpolates from $f=2$ (quadratic) to $f=\alpha^{(1)}$ (target adaptive regime), controlled by a continuation parameter $\beta$.
Details are provided in Appendix~\ref{realstep1solve}.

\section{Theoretical Properties}
We organize this section as follows. We first study the mean estimator with a fixed shape parameter $\alpha\in(0,1]$ and a sample-size--dependent scale $c_m$. We then establish two complementary regimes: (i) a finite-variance regime, in which the estimator is consistent and asymptotically normal at the classical $\sqrt{m}$ rate; and (ii) a heavy-tailed regime, where only a $(1+\epsilon)$-moment is assumed and the estimator attains a high-probability deviation bound of order $\bigl(\log(1/\delta)/m\bigr)^{\epsilon/(1+\epsilon)}$ for the single-block estimator in \eqref{generalform}. Finally, we study the asymptotic theory of the median-of-means aggregation---namely, the ARE in \eqref{ARE}---based on the collection of block means.
Throughout, the same estimating equation underlies all results; the difference between regimes is driven by how $c_m$ is scaled relative to the stochastic fluctuations of the empirical score.

Let $X_1,\dots,X_m$ be i.i.d. real-valued random variables.
The goal is to estimate the mean $\mu:=\E[X_1]$ (assumed to exist). Fix a shape parameter $\alpha\in(0,1]$ and let $\kappa_\alpha := 2-\alpha\in[1,2)$.
We use the adaptive robust loss family
\[
\rho(\varepsilon;\alpha,c)
:=
\frac{\kappa_\alpha}{\alpha}
\left[
\left(
1+\frac{(\varepsilon/c)^2}{\kappa_\alpha}
\right)^{\alpha/2}
-1
\right],
\qquad c>0,
\]
with score (influence) function
\[
\psi(\varepsilon;\alpha,c)
:=
\frac{\partial}{\partial \varepsilon}\rho(\varepsilon;\alpha,c)
=
\frac{\varepsilon}{c^2}
\left(
1+\frac{(\varepsilon/c)^2}{\kappa_\alpha}
\right)^{\alpha/2-1}.
\]
Heuristically, $\rho$ is locally quadratic while $\psi$ downweights large residuals when $|\varepsilon|$ exceeds the scale $c$.

Given a scale $c_m>0$ (possibly depending on $m$), define the empirical risk
\[
\widehat R_m(x):=\frac{1}{m}\sum_{i=1}^m \rho(X_i-x;\alpha,c_m),
\]
and the estimating equation
\[
\widehat g_m(x)
:=
\frac{d}{dx}\widehat R_m(x)
=
-\frac{1}{m}\sum_{i=1}^m \psi(X_i-x;\alpha,c_m).
\]
Let $R_m(x):=\E[\rho(X-x;\alpha,c_m)]$ and $g_m(x):=R_m'(x)=-\E[\psi(X-x;\alpha,c_m)]$.
Our proofs compare $\widehat g_m$ to its population counterpart $g_m$ on a neighborhood of $\mu$ to localize a stable root.

For $\alpha<1$, $\widehat g_m(x)=0$ may have multiple roots due to non-convexity.
We therefore define a \emph{central root} in a controlled interval.
Let $\widehat\mu_{\mathrm{pil}}$ be any pilot estimator and choose a radius $r_m>0$ such that $r_m\to\infty$ and $r_m=o(c_m)$.
Define
\(
\mathcal I_m := [\widehat\mu_{\mathrm{pil}}-r_m,\ \widehat\mu_{\mathrm{pil}}+r_m].
\)
\begin{definition}[Central-root estimator]
If the equation $\widehat g_m(x)=0$ has a unique solution in $\mathcal I_m$, denote it by $\widetilde X_m$.
\end{definition}
The only role of $\mathcal I_m$ is to rule out spurious distant roots while keeping $\mu$ inside with high probability.

\subsection{Finite-Variance Regime}
\begin{assumption}[Finite variance]
\label{ass:FV}
$\E[X_1]=\mu$ and $\E[(X_1-\mu)^2]=\sigma^2\in(0,\infty)$.
\end{assumption}

We use a purely asymptotic deterministic tuning: for a fixed $\gamma>0$, let
\(
c_m := m^{1/2+\gamma}, \quad r_m := m^{\gamma/4}.
\)
With this choice, $c_m$ diverges fast enough that the score is essentially linear on $\mathcal I_m$, recovering the classical $\sqrt{m}$ behavior.

\begin{theorem}
\label{thm:FV}
Assume Assumption~\ref{ass:FV} and fix $\alpha\in(0,1]$.
With the tuning above and with $\widetilde X_m$ defined as the unique central root of $\widehat g_m(x)=0$ in $\mathcal I_m$, there exists an event $\mathcal E_m$ with $\Pbb(\mathcal E_m)\to 1$ such that, on $\mathcal E_m$, $\widehat g_m$ has a unique zero in $\mathcal I_m$; moreover, $\widetilde X_m\xrightarrow{p}\mu$, and $\sqrt m(\widetilde X_m-\mu)\Rightarrow N(0,\sigma^2)$.
\end{theorem}


\begin{figure*}[htbp]
\centering
\includegraphics[width=0.95\textwidth]{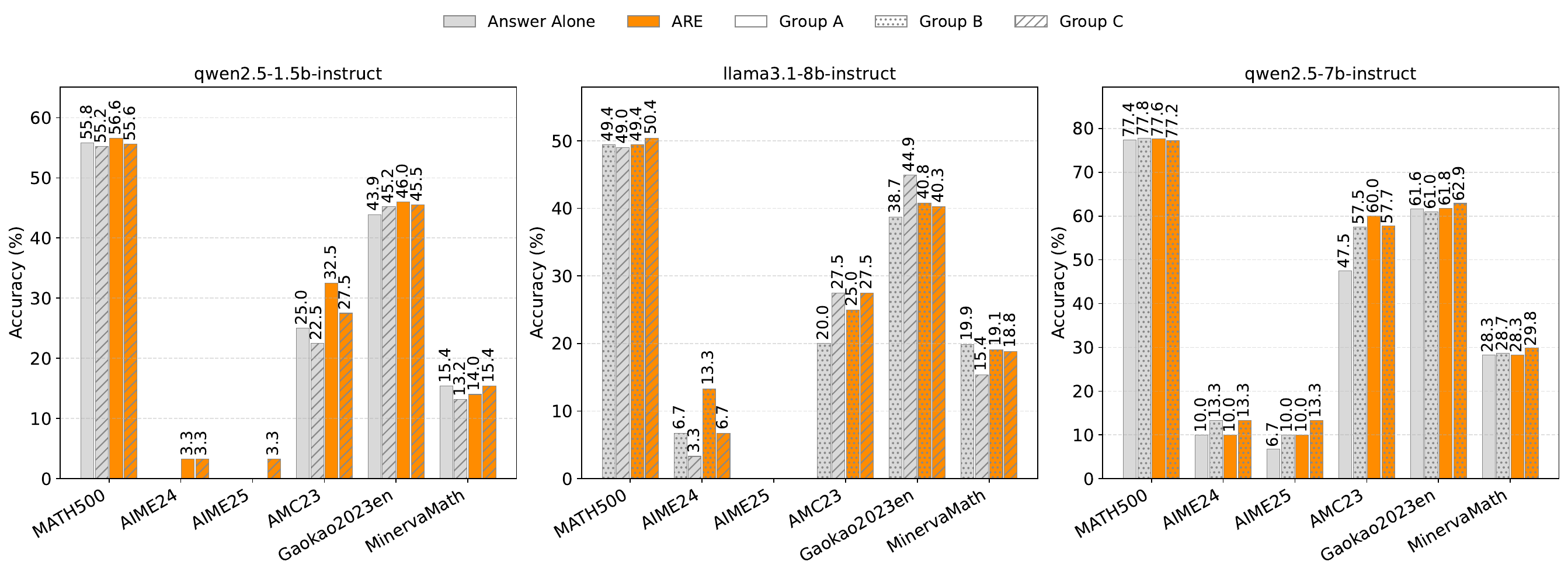}
\caption{Comparison of the effectiveness of the three-stage interaction model,The grey-filled area represents the experimental results without the three-stage interaction, while the orange part shows the results obtained with the three-stage interaction. It can be seen that the three-stage interactive solution is effective.}
\label{fig:effect_of_threeStage}
\end{figure*}

\begin{figure}[htbp]
\centering
\includegraphics[width=0.49\textwidth]{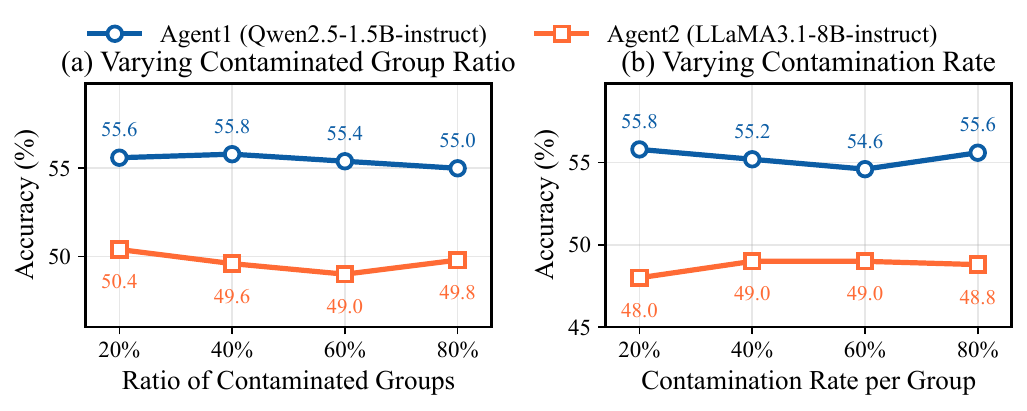}
\caption{The performance of the proposed method under different noise conditions is illustrated: the left figure shows contamination by noise groups of varying proportions, while the right figure depicts contamination ratios within the same group. It can be observed that the proposed method demonstrates strong noise resistance.}
\label{fig:noise_rate}
\vspace{-2mm}
\end{figure}

\begin{figure*}[htbp]
\centering
\includegraphics[width=1.0\textwidth]{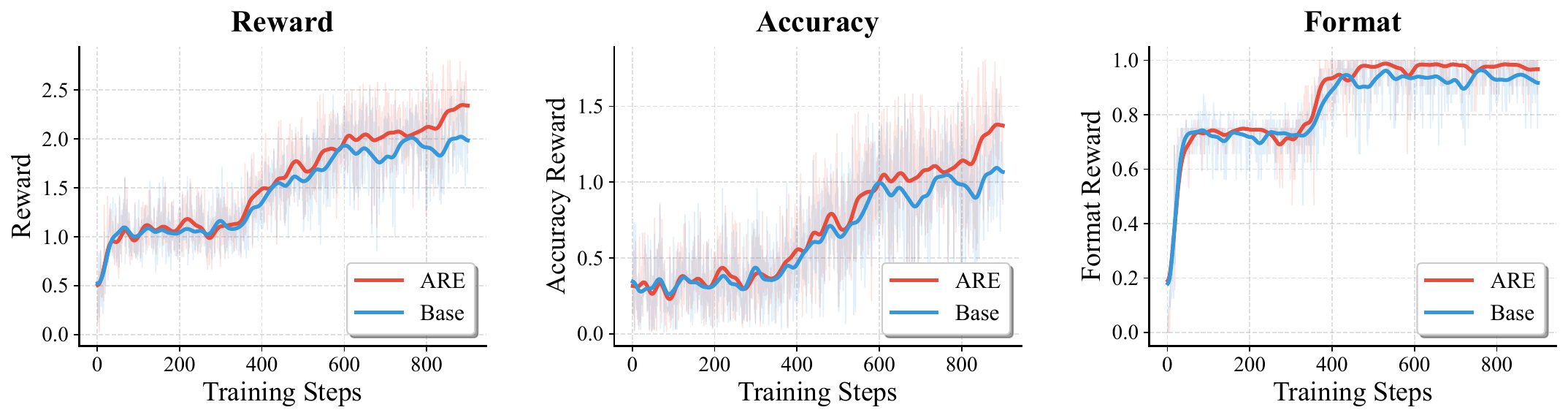}
\caption{Training reward curves of ARE and FlightGPT.}
\label{fig:ARE_train_curve}
\end{figure*}

\subsection{Heavy-Tailed Regime}
\begin{assumption}[$(1+\epsilon)$-moment]
\label{ass:HT}
Fix $\epsilon\in(0,1]$.
Assume $\E[X_1]=\mu$ and $\E[\abs{X_1-\mu}^{1+\epsilon}]\le v_{1+\epsilon}<\infty$.
\end{assumption}

For any confidence level $\delta\in(0,1/2)$, choose
\(
c_m(\delta):=\tau\left(\frac{v_{1+\epsilon}\,m}{\log(2/\delta)}\right)^{\frac{1}{1+\epsilon}}, \tau>0,
\)
and take any $r_m\to\infty$ with $r_m=o(c_m(\delta))$ (e.g., $r_m=c_m(\delta)^{1/2}$).
Here $c_m(\delta)$ is tuned to balance (controlled) bias from downweighting large residuals with concentration of the empirical score.

\begin{theorem}
\label{thm:HT}
Assume Assumption~\ref{ass:HT} and fix $\alpha\in(0,1]$.
With $c_m=c_m(\delta)$ and the central-root estimator $\widetilde X_m$, there exist constants $C_1,C_2>0$ depending only on $\alpha,\epsilon,\tau$ such that for all sufficiently large $m$,
\[
\Pbb\left(
\abs{\widetilde X_m-\mu}
\le
C_1\left(\frac{v_{1+\epsilon}\log(C_2/\delta)}{m}\right)^{\frac{\epsilon}{1+\epsilon}}
\right)\ge 1-\delta.
\]
\end{theorem}
This single-block deviation bound is the basic ingredient for median-of-means aggregation, which amplifies constant-success probability to an arbitrary confidence level.

\subsection{Asymptotic Theory for ARE}
We now aggregate across blocks using a sample median. In the finite-variance regime, a block-level CLT combined with the classical median CLT yields a $\sqrt{n}$ limit and the usual $\pi/2$ variance inflation, giving a constant asymptotic relative efficiency.

\begin{theorem}
\label{thm:ARE-FV}
Assume Assumption~\ref{ass:FV} and fix $\alpha\in(0,1]$.
Let $k=k_n\to\infty$ and $k_n=o(n)$, and set $m:=\lfloor n/k_n\rfloor\to\infty$.
On each block, compute $\widetilde X_j$ using the finite-variance tuning in
Theorem~\ref{thm:FV}, namely for a fixed $\gamma>0$,
\(
c_m := m^{1/2+\gamma},r_m:=m^{\gamma/4}.
\)
Assume in addition that the block-level distribution of $\widetilde X_j$ admits the local normal approximation
needed for the sample-median CLT (formalized in Appendix~\ref{app:ARE}).
Then
\[
\sqrt n\,(\widetilde\mu-\mu)\ \Rightarrow\ N\!\left(0,\frac{\pi}{2}\sigma^2\right).
\]
Consequently, relative to the sample mean $\bar X_n$, the asymptotic relative efficiency is
\(
\frac{2}{\pi}.
\)
\end{theorem}

\begin{theorem}[Heavy-tailed deviation bound]
\label{thm:ARE-HT}
Assume Assumption~\ref{ass:HT} with some $\epsilon\in(0,1]$ and fix $\alpha\in(0,1]$.
Let $\delta\in(0,1/2)$ and choose the number of blocks
\(
k:=\left\lceil 8\log\!\Big(\frac{2}{\delta}\Big)\right\rceil,
m:=\lfloor n/k\rfloor .
\)
On each block, run the central-root estimator $\widetilde X_j$ with the scale
\(
c_m:=c_m(\delta_0)
=
\tau\left(\frac{v_{1+\epsilon}\,m}{\log(2/\delta_0)}\right)^{\frac{1}{1+\epsilon}},
\delta_0:=\frac14,
\)
and any radius $r_m\to\infty$ such that $r_m=o(c_m)$ (e.g., $r_m=c_m^{1/2}$).
Then there exist constants $C_1,C_2>0$ depending only on $\alpha,\epsilon,\tau$ such that
for all sufficiently large $n$,
\[
\mathbb P\!\left(
\bigl|\widetilde\mu-\mu\bigr|
\le
C_1\left(\frac{v_{1+\epsilon}\log(C_2/\delta)}{n}\right)^{\frac{\epsilon}{1+\epsilon}}
\right)\ge 1-\delta.
\]
\end{theorem}

\section{Experiments}

\subsection{Mathematical Reasoning Experiments}
In this section, we evaluate the proposed ARE and DACR on challenging mathematical reasoning benchmarks across LLMs of different parameter scales and model families. We train on MATH~7.5k \citep{hendrycks2021measuringmathematicalproblemsolving} and test on MATH~500 \citep{lightman2023letsverifystepstep} (in-distribution), using AMC23, Gaokao2023en \citep{zhang2024evaluatingperformancelargelanguage}, and MinervaMath \citep{lewkowycz2022solvingquantitativereasoningproblems} as out-of-distribution suites to assess generalization. We compare against a baseline that matches our collaboration setting but replaces ARE/DACR with standard GRPO. We organize the empirical study around three research questions, examining overall effectiveness under noisy rewards, the contribution of the three-stage interaction, and robustness under diverse contamination patterns.

\textbf{Q1: Is the proposed multi-agent robust training effective?}\\
To evaluate robustness under noisy rewards, we inject heavy-tailed, peaked Cauchy noise into the reward model and train on mathematical reasoning datasets. As shown in \cref{table:maintable}, our method consistently improves accuracy over the GRPO-based baseline on both ID and OOD tests. The gains can be substantial; e.g., on AMC23, Qwen2.5-1.5B-Instruct improves from $22.5\%$ to $32.5\%$ (+$10\%$), suggesting the proposed multi-agent robust training is effective across model scales and families.

\textbf{Q2: Is the proposed three-stage interaction effective?}\\
We ablate the interaction protocol by training two agents that answer independently while keeping all other settings unchanged. Results in \cref{fig:effect_of_threeStage} show that the three-stage interaction yields consistently better reasoning performance across LLMs; in particular, on AMC23, Qwen2.5-7B-Instruct improves from $47.5\%$ to $60\%$ (+$12.5\%$).

\textbf{Q3: How robust is the method to noise?}\\
We further stress-test robustness under diverse contamination patterns (\cref{fig:noise_rate}): (i) group-level corruption where a fraction of groups are selected and 20\% of samples within each selected group are contaminated, and (ii) within-group corruption where the contamination ratio varies inside the same group. Both settings evaluate accuracy on MATH500. Across all conditions, our method degrades gracefully and remains consistently more robust to reward contamination than the baseline.

\begin{table*}[t]
\centering
\caption{Comparison of accuracy rates of various methods on the mathematical reasoning dataset}
\setlength{\tabcolsep}{5.7pt}        
\renewcommand{\arraystretch}{1} 
\begin{tabular}{c|c|cccccc}
\hline
\multirow{2}{*}{\textbf{Model}} &
\multirow{2}{*}{\textbf{Method}} &
\multirow{2}{*}{\textbf{MATH500}} &
\multirow{2}{*}{\textbf{AIME24}} &
\multirow{2}{*}{\textbf{AIME25}} &
\multirow{2}{*}{\textbf{AMC23}} &
\multirow{2}{*}{\shortstack[c]{\textbf{Gaokao}\\\textbf{2023en}}} &
\multirow{2}{*}{\shortstack[c]{\textbf{Minerva}\\\textbf{Math}}} \\
& & & & & & & \\   
\hline
\multirow{5}{*}{\shortstack[c]{Qwen2.5\\(1.5B-Instruct)}}
& Base Model & 35.80 & - & - & 30.00 & 23.50 & 5.150 \\ 
\cline{2-8}
& baseline-GroupA & 56.00 &- & - & 22.50 & 45.50 & 13.20 \\
& baseline-GroupC & 55.00 & \textbf{3.300} & - & 22.50 & 45.50 & 13.60 \\ 
\cline{2-8}
& \cellcolor{bg-green}ARE-GroupA & \cellcolor{bg-green}\textbf{56.60} & \cellcolor{bg-green}\textbf{3.300} & \cellcolor{bg-green}- & \cellcolor{bg-green}\textbf{32.50} & \cellcolor{bg-green}\textbf{46.00} & \cellcolor{bg-green}14.00 \\
& \cellcolor{bg-green}ARE-GroupC & \cellcolor{bg-green}55.60 & \cellcolor{bg-green}\textbf{3.300} & \cellcolor{bg-green}\textbf{3.300} & \cellcolor{bg-green}27.50 & \cellcolor{bg-green}45.50 & \cellcolor{bg-green}\textbf{15.40} \\
\hline

\multirow{5}{*}{\shortstack[c]{Llama3.1\\(8B-Instruct)}}
& Base Model & 32.60 & - & - & 5.000 & 3.390 & 7.350 \\ 
\cline{2-8}
& baseline-GroupB & 49.40 & 3.300 & - & 20.00 & 39.70 & 17.60 \\
& baseline-GroupC & 49.20 & 6.700 & 3.300 & 22.50 & 39.20 & 17.30 \\ 
\cline{2-8}
& \cellcolor{bg-green}ARE-GroupB & \cellcolor{bg-green}51.00 & \cellcolor{bg-green}\textbf{6.700} & \cellcolor{bg-green}- & \cellcolor{bg-green}22.50 & \cellcolor{bg-green}\textbf{42.60} & \cellcolor{bg-green}\textbf{19.90} \\
& \cellcolor{bg-green}ARE-GroupC & \cellcolor{bg-green}\textbf{50.40} & \cellcolor{bg-green}6.700 & \cellcolor{bg-green}- & \cellcolor{bg-green}\textbf{27.50} & \cellcolor{bg-green}40.30 & \cellcolor{bg-green}18.80 \\
\hline

\multirow{5}{*}{\shortstack[c]{Qwen2.5\\(7B-Instruct)}}
& Base Model & 68.00 & 6.667 & 10.00 & 37.50 & 42.04 & 18.38 \\ 
\cline{2-8}
& baseline-GroupA & \textbf{77.60} & \textbf{16.70} & 6.700 & 50.00 & \textbf{62.60} & 27.60 \\
& baseline-GroupB & 77.40 & 13.30 & 10.00 & 52.50 & 61.30 & 28.70 \\
\cline{2-8}
& \cellcolor{bg-green}ARE-GroupA & \cellcolor{bg-green}\textbf{77.60} & \cellcolor{bg-green}10.00 & \cellcolor{bg-green}10.00 & \cellcolor{bg-green}\textbf{60.00} & \cellcolor{bg-green}61.80 & \cellcolor{bg-green}28.30 \\
& \cellcolor{bg-green}ARE-GroupB & \cellcolor{bg-green}77.40 & \cellcolor{bg-green}13.30 & \cellcolor{bg-green}\textbf{10.00} & \cellcolor{bg-green}57.70 & \cellcolor{bg-green}61.60 & \cellcolor{bg-green}\textbf{30.50} \\
\hline
\end{tabular}
\label{table:maintable}
\end{table*}

\subsection{Embodied VLN Evaluation}
\begin{figure}[htbp]
\centering
\includegraphics[width=0.45\textwidth]{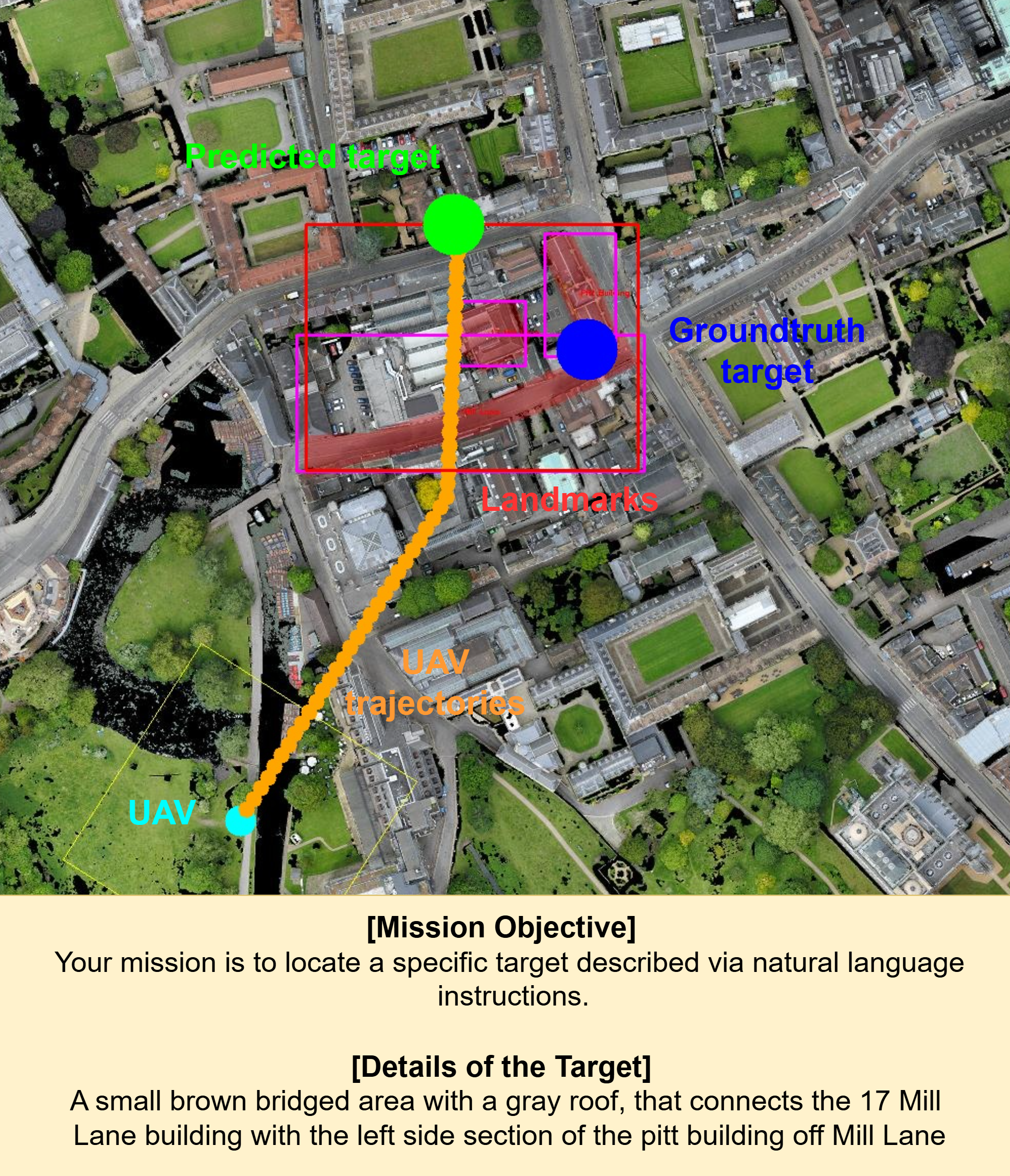}
\caption{Visualization of the navigation task. The UAV trajectory, predicted target, and ground-truth target are annotated.}
\label{fig:navigation}
\vspace{-2mm}
\end{figure}

To assess whether ARE transfers beyond pure language reasoning, we evaluated it on an embodied aerial VLN task by applying ARE to a visual language model. We built on FlightGPT~\cite{cai2025flightgpt} and replaced its standard GRPO optimization with ARE, fine-tuning a Qwen2-VL-2B-Instruct backbone \citep{Qwen2VL} with downscaled image resolutions to fit available compute; due to current constraints, we focused on validating the Adaptive Robust Estimator component, leaving full multi-agent interaction to future work. In VLN, the agent follows a natural-language instruction $D$ and visual observations from a 3D environment $E$, starting from $(\mathbf{p}_0,\theta_0)$ and executing discrete actions $a_t\in\mathcal{A}$; an episode succeeds if the final position $\mathbf{p}_T$ is within a threshold $\delta$ of the target.

Following~\cite{cai2025flightgpt}, we reported four standard metrics: Navigation Error (NE), the Euclidean distance between the final position and the ground-truth target (lower is better); Success Rate (SR), the fraction of episodes whose final position is within 20 meters of the target (higher is better); Oracle Success Rate (OSR), the fraction of episodes that reach within 20 meters of the target at any time (higher is better); and Success weighted by Path Length (SPL), which penalizes unnecessarily long trajectories (higher is better).

Figure~\ref{fig:ARE_train_curve} shows that ARE achieved consistently higher training rewards than FlightGPT, suggesting improved optimization stability under the same VLN setup; although we used a smaller backbone than in~\cite{cai2025flightgpt} (leading to lower absolute performance), the comparison to the GRPO baseline remained controlled. Consistently, Table~\ref{tab:vln_results_by_split} shows that ARE outperformed both the base model and FlightGPT across all splits, with the largest gains on unseen environments (e.g., NE 145.56$!\to$124.38 on $\text{val}_{\text{unseen}}$ and 140.78$!\to$129.71 on $\text{test}_{\text{unseen}}$) while also improving SR/OSR/SPL, indicating better localization, higher success, and more efficient trajectories.

\begin{table}[t]
\centering
\caption{Aerial VLN results (higher is better for SR/OSR/SPL; lower is better for NE).}
\label{tab:vln_results_by_split}
\footnotesize
\begin{threeparttable}
\begin{tabular}{llcccc}
\toprule
\textbf{Split} & \textbf{Method} & \textbf{NE} $\downarrow$ & \textbf{SR} $\uparrow$ & \textbf{OSR} $\uparrow$ & \textbf{SPL} $\uparrow$ \\
\midrule
\multirow{3}{*}{\textbf{VS}}
& Base  & 193.92 & 0.3300 & 1.0000 & 0.0000 \\
& FlightGPT & 148.19 & 2.6700 & 13.670 & 0.0200 \\
& \cellcolor{bg-green} ARE  & \cellcolor{bg-green} \textbf{139.49} & \cellcolor{bg-green} \textbf{3.0000} & \cellcolor{bg-green} \textbf{14.330} & \cellcolor{bg-green} \textbf{0.0300} \\
\midrule
\multirow{3}{*}{\textbf{VUS}}
& Base  & 190.84 & 1.3300 & 1.3300 & 0.0100 \\
& FlightGPT & 145.56 & 3.6700 & 18.000 & 0.0300 \\
& \cellcolor{bg-green} ARE  & \cellcolor{bg-green} \textbf{124.38} & \cellcolor{bg-green} \textbf{6.3300} & \cellcolor{bg-green} \textbf{18.330} & \cellcolor{bg-green} \textbf{0.0600} \\
\midrule
\multirow{3}{*}{\textbf{TUS}}
& Base  & 195.95 & 1.3300 & 2.3300 & 0.0100 \\
& FlightGPT & 140.78 & 2.3300 & 17.330 & 0.0200 \\
& \cellcolor{bg-green} ARE  & \cellcolor{bg-green} \textbf{129.71} & \cellcolor{bg-green} \textbf{4.6700} & \cellcolor{bg-green} \textbf{18.330} & \cellcolor{bg-green} \textbf{0.0400} \\
\bottomrule
\end{tabular}
\begin{tablenotes}[flushleft]
\footnotesize
\item VS: val\_seen, VUS: val\_unseen, TUS: test\_unseen.
\end{tablenotes}
\end{threeparttable}
\end{table}

\section{Conclusion}
This work addressed a key instability in MARL for LLM reasoning: noisy, heavy-tailed rewards can make GRPO’s batch-mean advantage normalization fragile, leading to oscillation or divergence. We proposed a robust collaborative training framework that coupled a structured DACR protocol with a cross-improvement reward for more reliable credit assignment, and introduced ARE to replace the batch mean in GRPO-style advantage estimation. We established theoretical guarantees for ARE under both finite-variance and heavy-tailed regimes, and experiments on mathematical reasoning benchmarks showed consistent gains in accuracy and training stability across model scales and families, including homogeneous and heterogeneous agent pairs, especially under heavy-tailed reward noise. We further validated ARE in an embodied aerial VLN setting by replacing FlightGPT’s GRPO with ARE on a Qwen2-VL-2B VLM, yielding higher training rewards and improved NE/SR/OSR/SPL across all splits, with the largest gains on unseen environments.

\section*{Impact Statement}
This paper presents work whose goal is to advance the field of machine learning by developing robust multi-agent policy optimization methods for collaborative reasoning under noisy, heavy-tailed rewards. We do not feel any potential societal consequences of this work necessary to be discussed here.

\bibliography{example_paper.bib}

\begin{thebibliography}{50}
\providecommand{\natexlab}[1]{#1}
\providecommand{\url}[1]{\texttt{#1}}
\expandafter\ifx\csname urlstyle\endcsname\relax
  \providecommand{\doi}[1]{doi: #1}\else
  \providecommand{\doi}{doi: \begingroup \urlstyle{rm}\Url}\fi

\bibitem[Barron(2019)]{BarronCVPR2019}
Barron, J.~T.
\newblock A general and adaptive robust loss function.
\newblock \emph{CVPR}, 2019.

\bibitem[Black \& Anandan(1996)Black and Anandan]{black1996robust}
Black, M.~J. and Anandan, P.
\newblock The robust estimation of multiple motions: Parametric and
  piecewise-smooth flow fields.
\newblock \emph{Computer vision and image understanding}, 63\penalty0
  (1):\penalty0 75--104, 1996.

\bibitem[Black \& Rangarajan(1996)Black and Rangarajan]{10.1007/BF00131148}
Black, M.~J. and Rangarajan, A.
\newblock On the unification of line processes, outlier rejection, and robust
  statistics with applications in early vision.
\newblock \emph{Int. J. Comput. Vision}, 19\penalty0 (1):\penalty0 57–91,
  July 1996.
\newblock ISSN 0920-5691.

\bibitem[Cai et~al.(2025)Cai, Dong, Tan, Deng, Li, Gao, Wang, Su, Sumalee, and
  Zhong]{cai2025flightgpt}
Cai, H., Dong, J., Tan, J., Deng, J., Li, S., Gao, Z., Wang, H., Su, Z.,
  Sumalee, A., and Zhong, R.
\newblock Flightgpt: Towards generalizable and interpretable uav
  vision-and-language navigation with vision-language models, 2025.
\newblock URL \url{https://arxiv.org/abs/2505.12835}.

\bibitem[Chebrolu et~al.(2021)Chebrolu, Läbe, Vysotska, Behley, and
  Stachniss]{9361339}
Chebrolu, N., Läbe, T., Vysotska, O., Behley, J., and Stachniss, C.
\newblock Adaptive robust kernels for non-linear least squares problems.
\newblock \emph{IEEE Robotics and Automation Letters}, 6\penalty0 (2):\penalty0
  2240--2247, 2021.

\bibitem[Chen et~al.(2025{\natexlab{a}})Chen, Ai, Li, Li, Wei, Zhou, Li, Yu,
  Chen, Sun, Zhuang, Li, Wang, and Ban]{chen2025llmboostmakelargelanguage}
Chen, Z., Ai, T., Li, Y., Li, G., Wei, Y., Zhou, W., Li, G., Yu, B., Chen, Z.,
  Sun, H., Zhuang, F., Li, J., Wang, D., and Ban, Y.
\newblock Llmboost: Make large language models stronger with boosting,
  2025{\natexlab{a}}.
\newblock URL \url{https://arxiv.org/abs/2512.22309}.

\bibitem[Chen et~al.(2025{\natexlab{b}})Chen, Ji, Mao, Wu, Cheng, Qin, Li, Li,
  Sun, Wang, et~al.]{chen2025scoring}
Chen, Z., Ji, Z., Mao, Q., Wu, H., Cheng, J., Qin, B., Li, Z., Li, J., Sun, K.,
  Wang, Z., et~al.
\newblock Scoring, reasoning, and selecting the best! ensembling large language
  models via a peer-review process.
\newblock \emph{arXiv preprint arXiv:2512.23213}, 2025{\natexlab{b}}.

\bibitem[Chen et~al.(2025{\natexlab{c}})Chen, Li, Chen, Li, Sun, Luo, Mao, Li,
  Xiao, Yang, et~al.]{chen2025harnessing}
Chen, Z., Li, J., Chen, P., Li, Z., Sun, K., Luo, Y., Mao, Q., Li, M., Xiao,
  L., Yang, D., et~al.
\newblock Harnessing multiple large language models: A survey on llm ensemble.
\newblock \emph{arXiv preprint arXiv:2502.18036}, 2025{\natexlab{c}}.

\bibitem[Chen et~al.(2026)Chen, Li, Ai, Li, Huang, Zhou, Zhuang, Liu, Li, Wang,
  et~al.]{chen2026weak}
Chen, Z., Li, G., Ai, T., Li, Y., Huang, Z., Zhou, W., Zhuang, F., Liu, X., Li,
  J., Wang, D., et~al.
\newblock Weak-driven learning: How weak agents make strong agents stronger.
\newblock \emph{arXiv preprint arXiv:2602.08222}, 2026.

\bibitem[Dennis~Jr \& Welsch(1978)Dennis~Jr and Welsch]{dennis1978techniques}
Dennis~Jr, J.~E. and Welsch, R.~E.
\newblock Techniques for nonlinear least squares and robust regression.
\newblock \emph{Communications in Statistics-simulation and Computation},
  7\penalty0 (4):\penalty0 345--359, 1978.

\bibitem[Geman \& Geman(1986)Geman and Geman]{geman1986bayesian}
Geman, D. and Geman, S.
\newblock Bayesian image analysis.
\newblock In \emph{Disordered systems and biological organization}, pp.\
  301--319. Springer, 1986.

\bibitem[Hastie et~al.(2001)Hastie, Tibshirani, and
  Friedman]{hastie01statisticallearning}
Hastie, T., Tibshirani, R., and Friedman, J.
\newblock \emph{The Elements of Statistical Learning}.
\newblock Springer Series in Statistics. Springer New York Inc., New York, NY,
  USA, 2001.

\bibitem[He et~al.(2025)He, Ban, Zou, Wei, Cook, and He]{he2025llm}
He, X., Ban, Y., Zou, J., Wei, T., Cook, C., and He, J.
\newblock Llm-forest: Ensemble learning of llms with graph-augmented prompts
  for data imputation.
\newblock In \emph{Findings of the Association for Computational Linguistics:
  ACL 2025}, pp.\  6921--6936, 2025.

\bibitem[Hendrycks et~al.(2021)Hendrycks, Burns, Kadavath, Arora, Basart, Tang,
  Song, and Steinhardt]{hendrycks2021measuringmathematicalproblemsolving}
Hendrycks, D., Burns, C., Kadavath, S., Arora, A., Basart, S., Tang, E., Song,
  D., and Steinhardt, J.
\newblock Measuring mathematical problem solving with the math dataset, 2021.
\newblock URL \url{https://arxiv.org/abs/2103.03874}.

\bibitem[Hitchcox \& Forbes(2022)Hitchcox and Forbes]{hitchcox2022mind}
Hitchcox, T. and Forbes, J.~R.
\newblock Mind the gap: Norm-aware adaptive robust loss for multivariate
  least-squares problems.
\newblock \emph{IEEE Robotics and Automation Letters}, 7\penalty0 (3):\penalty0
  7116--7123, 2022.

\bibitem[Hu(2025)]{hu2025reinforce++}
Hu, J.
\newblock Reinforce++: A simple and efficient approach for aligning large
  language models.
\newblock \emph{arXiv preprint arXiv:2501.03262}, 2025.

\bibitem[Huang et~al.(2026{\natexlab{a}})Huang, Xia, Ren, Zheng, Wang, Zhang,
  Xie, Liang, Chen, Xiao, et~al.]{huang2026does}
Huang, Z., Xia, X., Ren, Y., Zheng, J., Wang, X., Zhang, Z., Xie, H., Liang,
  S., Chen, Z., Xiao, X., et~al.
\newblock Does your reasoning model implicitly know when to stop thinking?
\newblock \emph{arXiv preprint arXiv:2602.08354}, 2026{\natexlab{a}}.

\bibitem[Huang et~al.(2026{\natexlab{b}})Huang, Xia, Ren, Zheng, Xiao, Xie,
  Huaqiu, Liang, Dai, Zhuang, et~al.]{huang2026real}
Huang, Z., Xia, X., Ren, Y., Zheng, J., Xiao, X., Xie, H., Huaqiu, L., Liang,
  S., Dai, Z., Zhuang, F., et~al.
\newblock Real-time aligned reward model beyond semantics.
\newblock \emph{arXiv preprint arXiv:2601.22664}, 2026{\natexlab{b}}.

\bibitem[Huber(2011)]{huber2011robust}
Huber, P.~J.
\newblock Robust statistics.
\newblock In \emph{International encyclopedia of statistical science}, pp.\
  1248--1251. Springer, 2011.

\bibitem[Humbert et~al.(2022)Humbert, Le~Bars, and
  Minvielle]{humbert2022robust}
Humbert, P., Le~Bars, B., and Minvielle, L.
\newblock Robust kernel density estimation with median-of-means principle.
\newblock In \emph{International Conference on Machine Learning}, pp.\
  9444--9465. PMLR, 2022.

\bibitem[Jung et~al.(2024)Jung, Hitchcox, and Forbes]{jung2024adaptive}
Jung, K., Hitchcox, T., and Forbes, J.~R.
\newblock An adaptive graduated nonconvexity loss function for robust nonlinear
  least-squares solutions.
\newblock \emph{IEEE Transactions on Robotics}, 2024.

\bibitem[Leclerc(1989)]{leclerc1989constructing}
Leclerc, Y.~G.
\newblock Constructing simple stable descriptions for image partitioning.
\newblock \emph{International journal of computer vision}, 3\penalty0
  (1):\penalty0 73--102, 1989.

\bibitem[Lecu{\'e} et~al.(2020)Lecu{\'e}, Lerasle, and
  Mathieu]{lecue2020robust}
Lecu{\'e}, G., Lerasle, M., and Mathieu, T.
\newblock Robust classification via mom minimization.
\newblock \emph{Machine learning}, 109\penalty0 (8):\penalty0 1635--1665, 2020.

\bibitem[Lewkowycz et~al.(2022)Lewkowycz, Andreassen, Dohan, Dyer, Michalewski,
  Ramasesh, Slone, Anil, Schlag, Gutman-Solo, Wu, Neyshabur, Gur-Ari, and
  Misra]{lewkowycz2022solvingquantitativereasoningproblems}
Lewkowycz, A., Andreassen, A., Dohan, D., Dyer, E., Michalewski, H., Ramasesh,
  V., Slone, A., Anil, C., Schlag, I., Gutman-Solo, T., Wu, Y., Neyshabur, B.,
  Gur-Ari, G., and Misra, V.
\newblock Solving quantitative reasoning problems with language models, 2022.
\newblock URL \url{https://arxiv.org/abs/2206.14858}.

\bibitem[Liao et~al.(2025)Liao, Wen, Wang, and Zhang]{liao2025marft}
Liao, J., Wen, M., Wang, J., and Zhang, W.
\newblock Marft: Multi-agent reinforcement fine-tuning.
\newblock \emph{arXiv preprint arXiv:2504.16129}, 2025.

\bibitem[Lightman et~al.(2023)Lightman, Kosaraju, Burda, Edwards, Baker, Lee,
  Leike, Schulman, Sutskever, and Cobbe]{lightman2023letsverifystepstep}
Lightman, H., Kosaraju, V., Burda, Y., Edwards, H., Baker, B., Lee, T., Leike,
  J., Schulman, J., Sutskever, I., and Cobbe, K.
\newblock Let's verify step by step, 2023.
\newblock URL \url{https://arxiv.org/abs/2305.20050}.

\bibitem[Lin et~al.(2025)Lin, Cao, Wang, Wu, Li, Yang, Zheng, and
  Qin]{lin2025interactive}
Lin, H., Cao, S., Wang, S., Wu, H., Li, M., Yang, L., Zheng, J., and Qin, C.
\newblock Interactive learning for llm reasoning.
\newblock \emph{arXiv preprint arXiv:2509.26306}, 2025.

\bibitem[Liu et~al.(2025{\natexlab{a}})Liu, Fu, Ding, Ning, Zhang, Liu, and
  Zhang]{liu2025logical}
Liu, H., Fu, Z., Ding, M., Ning, R., Zhang, C., Liu, X., and Zhang, Y.
\newblock Logical reasoning in large language models: A survey.
\newblock \emph{arXiv preprint arXiv:2502.09100}, 2025{\natexlab{a}}.

\bibitem[Liu et~al.(2025{\natexlab{b}})Liu, Chen, Liang, Lyu, and
  Amato]{liu2025llm}
Liu, S., Chen, T., Liang, Z., Lyu, X., and Amato, C.
\newblock Llm collaboration with multi-agent reinforcement learning.
\newblock \emph{arXiv preprint arXiv:2508.04652}, 2025{\natexlab{b}}.

\bibitem[Liu et~al.(2025{\natexlab{c}})Liu, Hu, Zhou, Ding, Li, Zeng, He, Chen,
  Jiang, Zhou, et~al.]{liu2025mathematical}
Liu, W., Hu, H., Zhou, J., Ding, Y., Li, J., Zeng, J., He, M., Chen, Q., Jiang,
  B., Zhou, A., et~al.
\newblock Mathematical language models: A survey.
\newblock \emph{ACM Computing Surveys}, 58\penalty0 (6):\penalty0 1--37,
  2025{\natexlab{c}}.

\bibitem[Lu et~al.(2022)Lu, Yu, de~Lamare, and Yang]{lu2022tukey}
Lu, L., Yu, Y., de~Lamare, R.~C., and Yang, X.
\newblock Tukey’s biweight m-estimate with conjugate gradient adaptive
  learning.
\newblock \emph{IEEE Signal Processing Letters}, 29:\penalty0 1117--1121, 2022.

\bibitem[Lu et~al.(2026)Lu, Wang, Chai, Yin, Lin, Chen, Luo, Zhuang, Ban, and
  Wang]{lu2026contextual}
Lu, X., Wang, X., Chai, J., Yin, G., Lin, W., Chen, Z., Luo, Y., Zhuang, F.,
  Ban, Y., and Wang, D.
\newblock Contextual rollout bandits for reinforcement learning with verifiable
  rewards.
\newblock \emph{arXiv preprint arXiv:2602.08499}, 2026.

\bibitem[Ma et~al.(2024)Ma, Hu, Pu, Boyin, Ai, Liang, and
  Chen]{ma2024coevolving}
Ma, H., Hu, T., Pu, Z., Boyin, L., Ai, X., Liang, Y., and Chen, M.
\newblock Coevolving with the other you: Fine-tuning llm with sequential
  cooperative multi-agent reinforcement learning.
\newblock \emph{Advances in Neural Information Processing Systems},
  37:\penalty0 15497--15525, 2024.

\bibitem[Motwani et~al.(2024)Motwani, Smith, Das, Rafailov, Laptev, Torr,
  Pizzati, Clark, and de~Witt]{motwani2024malt}
Motwani, S.~R., Smith, C., Das, R.~J., Rafailov, R., Laptev, I., Torr, P.~H.,
  Pizzati, F., Clark, R., and de~Witt, C.~S.
\newblock Malt: Improving reasoning with multi-agent llm training.
\newblock \emph{arXiv preprint arXiv:2412.01928}, 2024.

\bibitem[Park et~al.(2025)Park, Han, Guo, Ozdaglar, Zhang, and
  Kim]{park2025maporl}
Park, C., Han, S., Guo, X., Ozdaglar, A.~E., Zhang, K., and Kim, J.-K.
\newblock Maporl: Multi-agent post-co-training for collaborative large language
  models with reinforcement learning.
\newblock In \emph{Proceedings of the 63rd Annual Meeting of the Association
  for Computational Linguistics (Volume 1: Long Papers)}, pp.\  30215--30248,
  2025.

\bibitem[Rafailov et~al.(2023)Rafailov, Sharma, Mitchell, Manning, Ermon, and
  Finn]{rafailov2023direct}
Rafailov, R., Sharma, A., Mitchell, E., Manning, C.~D., Ermon, S., and Finn, C.
\newblock Direct preference optimization: Your language model is secretly a
  reward model.
\newblock \emph{Advances in neural information processing systems},
  36:\penalty0 53728--53741, 2023.

\bibitem[Ren et~al.(2025)Ren, Jiang, Yang, Tian, and Peng]{ren2025riskpo}
Ren, T., Jiang, J., Yang, H., Tian, W., and Peng, Y.
\newblock Risk{PO}: Risk-based policy optimization with verifiable reward for
  {LLM} post-training.
\newblock In \emph{NeurIPS 2025 Workshop MLxOR: Mathematical Foundations and
  Operational Integration of Machine Learning for Uncertainty-Aware
  Decision-Making}, 2025.
\newblock URL \url{https://openreview.net/forum?id=8hxqmh25ZH}.

\bibitem[Rieder(2012)]{rieder2012robust}
Rieder, H.
\newblock \emph{Robust Asymptotic Statistics: Volume I}.
\newblock Springer Science \& Business Media, 2012.

\bibitem[Shao et~al.(2024)Shao, Wang, Zhu, Xu, Song, Bi, Zhang, Zhang, Li, Wu,
  and Guo]{shao2024deepseekmathpushinglimitsmathematical}
Shao, Z., Wang, P., Zhu, Q., Xu, R., Song, J., Bi, X., Zhang, H., Zhang, M.,
  Li, Y.~K., Wu, Y., and Guo, D.
\newblock Deepseekmath: Pushing the limits of mathematical reasoning in open
  language models, 2024.
\newblock URL \url{https://arxiv.org/abs/2402.03300}.

\bibitem[Shao et~al.(2025)Shao, Luo, Lu, Ren, Hu, Ye, Gou, Ma, and
  Zhang]{shao2025deepseekmath}
Shao, Z., Luo, Y., Lu, C., Ren, Z., Hu, J., Ye, T., Gou, Z., Ma, S., and Zhang,
  X.
\newblock Deepseekmath-v2: Towards self-verifiable mathematical reasoning.
\newblock \emph{arXiv preprint arXiv:2511.22570}, 2025.

\bibitem[Sun(2021)]{sun2021we}
Sun, Q.
\newblock Do we need to estimate the variance in robust mean estimation?
\newblock \emph{arXiv preprint arXiv:2107.00118}, 2021.

\bibitem[Wan et~al.(2025)Wan, Li, Wen, Song, Wang, Yang, Schmidt, Wang, Zhang,
  Hu, et~al.]{wan2503rema}
Wan, Z., Li, Y., Wen, X., Song, Y., Wang, H., Yang, L., Schmidt, M., Wang, J.,
  Zhang, W., Hu, S., et~al.
\newblock Rema: Learning to meta-think for llms with multi-agent reinforcement
  learning.
\newblock \emph{arXiv preprint arXiv:2503.09501}, 2025.

\bibitem[Wang et~al.(2024)Wang, Bai, Tan, Wang, Fan, Bai, Chen, Liu, Wang, Ge,
  Fan, Dang, Du, Ren, Men, Liu, Zhou, Zhou, and Lin]{Qwen2VL}
Wang, P., Bai, S., Tan, S., Wang, S., Fan, Z., Bai, J., Chen, K., Liu, X.,
  Wang, J., Ge, W., Fan, Y., Dang, K., Du, M., Ren, X., Men, R., Liu, D., Zhou,
  C., Zhou, J., and Lin, J.
\newblock Qwen2-vl: Enhancing vision-language model's perception of the world
  at any resolution.
\newblock \emph{arXiv preprint arXiv:2409.12191}, 2024.

\bibitem[Wang et~al.(2025)Wang, Liu, Wang, Li, Wang, Yan, Jia, Liu, Chen, Xu,
  et~al.]{wang2025survey}
Wang, P.-Y., Liu, T.-S., Wang, C., Li, Z., Wang, Y., Yan, S., Jia, C., Liu,
  X.-H., Chen, X., Xu, J., et~al.
\newblock A survey on large language models for mathematical reasoning.
\newblock \emph{ACM Computing Surveys}, 2025.

\bibitem[Wang \& Xu(2025)Wang and Xu]{wang2025multi}
Wang, X. and Xu, M.
\newblock Multi-agent multi-armed bandit with fully heavy-tailed dynamics.
\newblock \emph{arXiv preprint arXiv:2501.19239}, 2025.

\bibitem[Xue et~al.(2025)Xue, Zhou, Zhang, Zhang, Li, Zhang, Yin, Torr, Ouyang,
  and Bai]{xue2025comas}
Xue, X., Zhou, Y., Zhang, G., Zhang, Z., Li, Y., Zhang, C., Yin, Z., Torr, P.,
  Ouyang, W., and Bai, L.
\newblock Comas: Co-evolving multi-agent systems via interaction rewards.
\newblock \emph{arXiv preprint arXiv:2510.08529}, 2025.

\bibitem[Yang et~al.(2026)Yang, Chen, Wang, Lu, Chai, Yin, Lin, Ma, Zhuang,
  Wang, et~al.]{yang2026grouprelativeadvantagebiased}
Yang, F., Chen, Z., Wang, X., Lu, X., Chai, J., Yin, G., Lin, W., Ma, S.,
  Zhuang, F., Wang, D., et~al.
\newblock Your group-relative advantage is biased.
\newblock \emph{arXiv preprint arXiv:2601.08521}, 2026.

\bibitem[Zhang et~al.(2024)Zhang, Li, Zong, Ying, He, and
  Qiu]{zhang2024evaluatingperformancelargelanguage}
Zhang, X., Li, C., Zong, Y., Ying, Z., He, L., and Qiu, X.
\newblock Evaluating the performance of large language models on gaokao
  benchmark, 2024.
\newblock URL \url{https://arxiv.org/abs/2305.12474}.

\bibitem[Zhang et~al.(2026)Zhang, Huang, Xia, Wang, Zhuang, Ma, Ding, Yang, Li,
  and Ban]{zhang2026heterogeneous}
Zhang, Z., Huang, Z., Xia, X., Wang, D., Zhuang, F., Ma, S., Ding, N., Yang,
  Y., Li, J., and Ban, Y.
\newblock Heterogeneous agent collaborative reinforcement learning.
\newblock \emph{arXiv preprint arXiv:2603.02604}, 2026.

\bibitem[Zou et~al.(2025)Zou, Ban, Li, Qi, Qiu, Yang, and
  He]{zou2025transformer}
Zou, J., Ban, Y., Li, Z., Qi, Y., Qiu, R., Yang, L., and He, J.
\newblock Transformer copilot: Learning from the mistake log in {LLM}
  fine-tuning.
\newblock In \emph{The Thirty-ninth Annual Conference on Neural Information
  Processing Systems}, 2025.
\newblock URL \url{https://openreview.net/forum?id=MRvxlTlkNQ}.

\end{thebibliography}
\bibliographystyle{icml2026}

\newpage
\appendix
\onecolumn

\section{Related Work}

\subsection{Multi-Agent Collaboration and Training}

Recent advances in large language models have inspired substantial interest in multi-agent collaboration\cite{chen2025llmboostmakelargelanguage,he2025llm,chen2026weak,chen2025scoring,lu2026contextual}, where multiple LLM-based agents interact to solve tasks that exceed the capabilities of a single model. Early explorations primarily focused on role specialization, task decomposition, and collaborative protocols such as planning, debating, and mutual verification. These frameworks improve solution quality through interaction but typically rely on fixed, frozen models without joint optimization. More recent studies extend this paradigm by treating multi-agent systems as interactive learning environments, in which agents coordinate, compete, or exchange information while being jointly optimized through reinforcement learning. Collectively, these efforts demonstrate that coordinated multi-agent LLMs can exhibit stronger reasoning, more robust generalization, and higher reliability than single-agent systems.

A representative example is CORY \citep{ma2024coevolving}, which formulates two symbiotic agents: a pioneer and an observer, both initialized from the same base model. The pioneer generates an initial response, while the observer synthesizes an improved refinement based on the pioneer’s output. The two agents periodically switch roles to foster cooperative evolution. MARFT \citep{liao2025marft} models multi-agent collaboration as a sequential decision-making process that aligns with the PPO training paradigm. This approach allows agents to share a common backbone while adopting distinct capabilities through different LoRA adapters.

Beyond these architectures, several works explicitly incorporate reinforcement learning for multi-agent reasoning. \citet{lin2025interactive} propose an interactive learning framework where agents dynamically choose cooperative or competitive behaviors based on problem difficulty and model capability. Their three-stage discussion protocol enhances information exchange, and the reward design incorporates distributional properties of other agents’ returns. Training is performed using GRPO, ultimately improving the reasoning ability of individual agents through multi-agent synergy. Similarly, CoMAS \citep{xue2025comas} constructs a rich interaction loop based on solution proposal, peer evaluation, and scoring, and employs an LLM-as-a-judge mechanism to convert discussion signals into reward feedback. These rewards are then optimized via the REINFORCE++ \citep{hu2025reinforce++} algorithm, enabling decentralized and scalable emergent learning without external supervision.\cite{yang2026grouprelativeadvantagebiased} provides a principled theoretical analysis of group-based advantage estimation.

Several recent works extend GRPO into explicit multi-agent regimes. ReMA \citep{wan2503rema} introduces meta-level reasoning into hierarchical multi-agent training, involving high-level and low-level agents trained under both shared and independent GRPO setups with single- and multi-round interaction schemes. MAGRPO \citep{liu2025llm} further generalizes GRPO to collaborative fine-tuning across multiple agents, where a centralized group-relative advantage guides joint optimization while retaining decentralized execution; this method has demonstrated effectiveness in program synthesis tasks. MAPoRL \citep{park2025maporl} proposes a post-training framework that combines multi-agent reinforcement learning with a carefully designed reward mechanism to explicitly train LLMs for coordinated decision-making in complex environments. Meanwhile, MALT \cite{motwani2024malt} constructs a division-of-labor multi-agent system with iterative collaboration and experience-based improvement, leveraging tree-structured data generation, value-iteration-style credit assignment, and a hybrid SFT–DPO \citep{rafailov2023direct} post-training pipeline to overcome the limitations of single-agent reasoning.

While these works have established a solid foundation for multi-agent reinforcement learning with LLMs, they rarely address the robustness challenges of GRPO in multi-agent optimization. In contrast, our approach introduces a new collaboration mechanism and explicitly incorporates robust reinforcement learning techniques to stabilize multi-agent GRPO training.

\subsection{Robust Estimation Methods}
Robust estimation methods are designed to maintain statistical reliability when data deviate from idealized assumptions, such as in the presence of outliers, model misspecification, or heavy-tailed noise \citep{huber2011robust,rieder2012robust}. Classical estimators—e.g., the sample mean or ordinary least squares—can be highly sensitive to a small fraction of corrupted observations, leading to unstable inference and degraded predictive performance. After decades of development, robust estimation methods can be categorized—according to their evolutionary trajectory—into the design of robust loss functions, adaptive procedures, tuning-free approaches, and median-of-means (MOM) estimators.

The core idea of robust loss functions is to minimize the influence of abnormally large per-sample residuals or gradients on the overall objective \citep{jung2024adaptive}. Representative choices include the Huber \citep{huber2011robust}, pseudo-Huber \citep{hastie01statisticallearning}, Cauchy \citep{black1996robust}, Tukey's biweight \citep{lu2022tukey}, and Geman--McClure (GM) \citep{geman1986bayesian} losses . A key caveat, however, is the risk of over-robustness: overly aggressive down-weighting of extreme observations may suppress rare yet genuine signals. Moreover, the landscape of robust losses is broad, and selecting an appropriate form typically relies heavily on expert knowledge and empirical tuning. To mitigate this issue, \citet{BarronCVPR2019} proposed an adaptive loss (as shown in (\ref{adaptiveloss})) which, via a shape parameter $\alpha$ governing tail behavior and robustness and a scale parameter $c$ controlling the width of the near-zero quadratic basin, unifies many classical losses within a single continuously adjustable---and potentially automatically learnable---framework. Nevertheless, both classes of approaches may suffer from substantial computational overhead for tuning-parameter selection (e.g., common choices such as Lepski's method and cross-validation can be computationally expensive), and the resulting objectives are often nonconvex \citep{jung2024adaptive}, which can lead to poor local optima, sensitivity to initialization, and training instability.

To address these two major challenges, a third line of work---tuning-free methods---has emerged. The key idea is to determine the value of the tuning parameter (typically depending on the target estimation accuracy and the population second moment of the data) under which the resulting estimator can achieve a sub-Gaussian rate of convergence. A representative example is the work of \citet{sun2021we}, which builds on the pseudo-Huber loss and derives an explicit criterion for choosing the tuning parameter. A fourth class of approaches that has attracted considerable attention in recent years is the MoM estimator \citep{lecue2020robust, humbert2022robust}. The procedure first partitions the sample into multiple groups, computes the sample mean within each group, and then takes the median of these group-wise means. Under the mere existence of a finite second moment, the MoM estimator can achieve sub-Gaussian rates of convergence. In Appendix \ref{MOMDiscussion}, we further discuss that, under both random contamination and adversarial contamination, the MoM estimator may still fail to be robust; moreover, attaining strong empirical performance typically requires a careful trade-off between the number of blocks and the sample size within each block.

\section{Discussion on the MOM Estimator and its Robustness}\label{MOMDiscussion}
In this section, we review the Median-of-Means (MoM) estimator as our baseline and discuss its key limitations under two contamination regimes: \emph{random} and \emph{adversarial} contamination.

Let \(X_1,\dots,X_n\) be i.i.d.\ real-valued random variables with mean \(\mu=\mathbb{E}[X_i]\). The empirical mean \(\overline X=\frac{1}{n}\sum_{i=1}^n X_i\) is statistically optimal under light-tailed distributions (e.g., sub-Gaussian), yet it can be highly unstable in the presence of heavy tails or contaminated samples. As noted in Section~\ref{secintro}, MoM provides a canonical robust alternative by combining blockwise averaging with a robust aggregation step. Concretely, fix \(k\ge 2\) and randomly partition \([n]\) into \(k\) disjoint blocks \(B_1,\dots,B_k\) of size \(m=\lfloor n/k\rfloor\). Define the block means
\[
\overline{X}_{j} = \frac{1}{|B_j|} \sum_{i \in B_j} X_i, \qquad j = 1, \dots, k,
\]
and aggregate them via the median:
\[
\widehat{\mu} = \operatorname{median}(\overline{X}_{1}, \dots, \overline{X}_{k}).
\]

This construction exhibits a transparent robustness mechanism: blockwise averaging reduces the variance contributed by the uncontaminated observations, while the median discards the influence of a minority of corrupted blocks. In particular, under the mild assumption \(\sigma^2=\mathrm{Var}(X_i)<\infty\), MoM achieves a sub-Gaussian-type concentration bound. For any \(\delta\in(0,1)\), choosing \(k\simeq c_0\log(1/\delta)\) yields
\[
\Pbb\!\left( \bigl| \widehat{\mu} - \mu \bigr| \le C_0 \sigma \sqrt{\frac{\log(1/\delta)}{n}} \right) \ge 1 - \delta,
\]
for universal constants \(c_0,C_0>0\) \citep{sun2021we}. Hence, even without sub-Gaussian tails, MoM retains the optimal \(n^{-1/2}\) rate and the canonical \(\sqrt{\log(1/\delta)}\) confidence scaling. Practically, \(k\) controls the robustness--efficiency trade-off: increasing \(k\) improves tolerance to corruption through the median at the cost of smaller block sizes. A common default is \(k=\lceil c_0\log(1/\delta)\rceil\) (with \(c_0\) in the range \(2\)--\(8\)); when \(\delta\) is not specified, the heuristic \(k\approx\lceil \log n\rceil\) is often used. We now turn to a simple analysis and discussion.

\textbf{In random contamination,} assume each observation is independently corrupted with probability $\kappa\in(0,1)$. Randomly partition the $n$ samples into $k$ disjoint blocks, each of size
\(
m=\lfloor n/k\rfloor.
\)
Call a block \emph{clean} if none of its $m$ samples is contaminated. Under independent contamination, the probability that a fixed block is clean is
\[
p_0=\Pr(\text{a block is clean})=(1-\kappa)^m,
\]
where we ignore the minor imbalance due to the floor operation and possible leftovers. Let $X_j=\mathbf{1}\{\text{block $j$ is clean}\}$ for $j=1,\dots,k$. Then $X_j\sim \mathrm{Bernoulli}(p_0)$ i.i.d., and the number of clean blocks is
\[
S:=\sum_{j=1}^k X_j \sim \mathrm{Binomial}(k,p_0).
\]
A key mechanism behind MoM is that the median is taken over a \emph{majority} of blocks; thus, if more than half of the blocks are clean (i.e., $S>k/2$), then the median is computed from predominantly uncontaminated block means. This motivates the condition $p_0>1/2$. When $p_0>1/2$, a Chernoff--Hoeffding (KL-form) bound yields
\[
\Pr(S\le k/2)
\le 
\exp\!\Bigl(-k\,\mathrm{KL}\bigl(1/2\,\|\,p_0\bigr)\Bigr),
\]
where the binary KL divergence is
\[
\mathrm{KL}\bigl(1/2\,\|\,p_0\bigr)
=
\frac12 \ln\!\frac{1/2}{p_0}
+
\frac12 \ln\!\frac{1/2}{1-p_0}.
\]
Hence, provided $p_0>1/2$, the failure event ``at most half of the blocks are clean'' has probability exponentially small in $k$.

The condition $p_0>1/2$ is equivalent to
\[
(1-\kappa)^m>\frac12
\quad\Longleftrightarrow\quad
m<\frac{\ln 2}{-\ln(1-\kappa)}.
\]
For small $\kappa$, using $-\ln(1-\kappa)\approx \kappa$, this becomes the convenient rule of thumb
\[
m \lesssim \frac{\ln 2}{\kappa}.
\]
\textcolor{blue}{Thus, smaller blocks (smaller $m$, larger $k$) make it easier to ensure that a majority of blocks are clean, which is precisely the regime in which MoM can retain sub-Gaussian-type deviation bounds under only finite second moments.}

It is important to note that ``smaller blocks are cleaner'' does not mean ``the smaller the better.'' Decreasing $m$ increases $p_0=(1-\kappa)^m$ and improves the probability of having a clean majority, but it also increases the variance of each within-block mean because fewer samples are averaged within a block. Therefore, obtaining good estimation accuracy requires balancing the number of blocks $k$ and the block size $m$.


\textbf{In the adversarial contamination,} an attacker can strategically place corrupted observations across blocks. In particular, it suffices to inject a single extremely large outlier into each of $\lceil k/2\rceil$ distinct blocks to arbitrarily shift the MoM output. Indeed, within any contaminated block of size $m$, the block mean takes the form
\[
\bar X_{\text{block}}=\frac{1}{m}\sum_{i=1}^m X_i,
\]
so a single outlier of magnitude $M$ shifts the block mean by
\[
\Delta=\frac{M}{m}.
\]
Although the multiplicative factor $1/m$ attenuates the effect, the attacker may choose $M$ arbitrarily large, hence $\Delta$ can be made arbitrarily large. Consequently, once at least half of the blocks are contaminated, the median over block means can be driven to an arbitrary value. This yields a finite-sample breakdown point on the order of
\[
\epsilon_{\mathrm{bd}} \approx \frac{\lceil k/2\rceil}{n},
\]
since contaminating $\lceil k/2\rceil$ blocks requires only $\lceil k/2\rceil$ adversarial points among $n$ observations. \textcolor{red}{In particular, if $k$ is fixed (or grows sublinearly in $n$), then $\epsilon_{\mathrm{bd}}\to 0$ as $n\to\infty$, indicating extremely poor robustness in the strict breakdown-point sense.}
Equivalently, even when the overall contamination proportion is very small, MoM can be completely compromised if more than half of the blocks each contain at least one outlier.

Finally, the choice of $k$ exhibits an intrinsic tension. If $k$ is too large, then $m=\lfloor n/k\rfloor$ becomes too small and the variance of each block mean inflates; if $k$ is too small, then blocks are too large and it becomes easier for contamination to hit a majority of blocks (especially under random contamination). \textcolor{red}{Therefore, MoM entails a structural trade-off between $k$ and $m$ that is difficult to calibrate in practice.}

\section{Simulation Studies of the ARE}
In the main text, we validate the effectiveness of the ARE estimator through several concrete LLM applications. In this section, we further demonstrate its performance via simulation studies. 

We consider one-dimensional robust mean estimation with ground-truth mean $\mu^\star=1$. For each sample size $n\in\{200,500,1000,2000,5000\}$, we generate i.i.d.\ samples $X_1,\dots,X_n$ under five data-generating processes (DGP1--DGP5) that progressively increase distributional difficulty from light tails to heavy tails and explicit corruption. DGP1 (Gaussian) serves as a light-tailed baseline:
\[
X_i=\mu^\star+\sigma Z_i,\qquad Z_i\sim\mathcal{N}(0,1),
\]
with default $\sigma=1$. DGP2 (centered log-normal) captures skewed heavy tails while preserving the mean: letting $Y_i\sim\mathcal{N}(0,1)$,
\[
X_i=\mu^\star+\Bigl(\exp(\tau Y_i)-\exp(\tau^2/2)\Bigr),
\]
where $\tau>0$ controls tail-heaviness and skewness (default $\tau=1.5$). DGP3 (Student-$t$) models symmetric heavy tails with finite variance:
\[
X_i=\mu^\star+sT_i,\qquad T_i\sim t_\nu,\qquad \nu>2,
\]
with default $(s,\nu)=(1,4)$; in our implementation $T_i=Z_i/\sqrt{V_i/\nu}$ with $Z_i\sim\mathcal{N}(0,1)$ and $V_i\sim\chi^2(\nu)$. DGP4 (symmetric Pareto) probes the infinite-variance regime while keeping the mean finite:
\[
X_i=\mu^\star+S_iU_i,
\]
where $S_i\in\{-1,+1\}$ is Rademacher (equiprobable) and $U_i\ge 1$ has Pareto tail $\mathbb{P}(U_i>u)=u^{-\alpha}$ for $u\ge 1$, with $\alpha\in(1,2)$ (default $\alpha=1.5$). Finally, DGP5 (contamination) explicitly injects outliers on top of a clean baseline $X_i^{(0)}\sim\mathcal{N}(\mu^\star,\sigma_0^2)$ with default $\sigma_0=1$. Under random (mean-preserving) contamination, independently for each $i$ we set $X_i=X_i^{(0)}$ with probability $1-\kappa$, and with probability $\kappa$ replace
\[
X_i\leftarrow \mu^\star+S_iM,
\]
where $S_i\in\{-1,+1\}$ is equiprobable (default $(\kappa,M)=(0.05,100)$). Under adversarial one-sided contamination, after drawing the clean sample we replace exactly $\lfloor\kappa n\rfloor$ points by the outlier $\mu^\star+M$. To further stress-test median-of-means (MoM) style aggregation, we also consider a block-aware adversary: we deterministically partition the first $mk$ samples into $k$ blocks of equal size $m=\lfloor n/k\rfloor$ (discarding leftovers), then select $r=\lceil k/2\rceil$ blocks and inject one outlier $\mu^\star+M$ into each selected block (optionally capped by a global budget $\lfloor\kappa n\rfloor$).

We compare the estimators \texttt{l2}, \texttt{Cauchy}, \texttt{GM}, \texttt{Adapt}, \texttt{AMB} (adaptive Maxwell--Boltzmann \citep{hitchcox2022mind}), their GNC variants \texttt{GNC\_Adapt}, \texttt{GNC\_AMB}, \texttt{GNC\_TLS} (truncated least squares), and the MoM-integrated variant \texttt{ARE}. Each method is initialized at the sample mean $\mu^{(0)}=\frac{1}{n}\sum_{i=1}^n X_i$ and returns a scalar estimate $\widehat{\mu}$. For each pair $(\mathrm{DGP},n)$, we run $R=100$ Monte Carlo replications. To ensure fair comparisons, we use common random numbers: in replication $r\in\{0,\dots,R-1\}$ we set the seed to $\texttt{base\_seed}+r$ (default $\texttt{base\_seed}=12345$), generate one dataset, and evaluate all estimators on the same data. Performance is summarized by the error $e=\widehat{\mu}-\mu^\star$ through MAE and MSE,
\[
\mathrm{MAE}=\mathbb{E}[|e|],\qquad \mathrm{MSE}=\mathbb{E}[e^2],
\]
together with robust quantiles of $|e|$ (Median, Q90, and Q95 computed via linear interpolation on sorted absolute errors). Tables~\ref{tab:dgp1_gaussian}--\ref{tab:dgp5_contamination_random} report the experimental results for DGP1--DGP5 across all sample sizes and estimators.

\begin{table*}[h]
\centering
\caption{Results on DGP1 (Gaussian): robust error summaries over $R=100$ replications.}
\label{tab:dgp1_gaussian}
\scriptsize
\setlength{\tabcolsep}{4pt}
\begin{tabular}{c|cccccccccc}
\hline
$n$ & Metric & l2 & Cauchy & GM & Adapt & AMB & GNC\_Adapt & GNC\_AMB & GNC\_TLS & ARE \\
\hline
\multirow{5}{*}{200}
& MAE    & 0.057287 & 0.074298 & 0.082726 & 0.089879 & 0.086309 & 0.089576 & 0.086304 & 0.067855 & 0.093291 \\
& MSE    & 0.005173 & 0.008747 & 0.010770 & 0.012688 & 0.011817 & 0.012600 & 0.011710 & 0.007302 & 0.013051 \\
& Median & 0.047519 & 0.062271 & 0.068645 & 0.074465 & 0.070920 & 0.074508 & 0.072363 & 0.056955 & 0.080019 \\
& Q90    & 0.117608 & 0.152255 & 0.173105 & 0.185062 & 0.181983 & 0.185578 & 0.178264 & 0.140570 & 0.187690 \\
& Q95    & 0.137931 & 0.199400 & 0.216909 & 0.229500 & 0.224137 & 0.225030 & 0.220215 & 0.163779 & 0.224390 \\
\hline
\multirow{5}{*}{500}
& MAE    & 0.034494 & 0.044633 & 0.049594 & 0.053918 & 0.052523 & 0.053909 & 0.052433 & 0.041655 & 0.063132 \\
& MSE    & 0.001838 & 0.002937 & 0.003645 & 0.004326 & 0.004069 & 0.004322 & 0.004044 & 0.002647 & 0.006196 \\
& Median & 0.029021 & 0.040605 & 0.043591 & 0.046434 & 0.045928 & 0.046208 & 0.046618 & 0.035880 & 0.052503 \\
& Q90    & 0.069328 & 0.089981 & 0.100150 & 0.106283 & 0.105652 & 0.106068 & 0.105136 & 0.085231 & 0.132819 \\
& Q95    & 0.081448 & 0.099868 & 0.114402 & 0.126156 & 0.119561 & 0.126863 & 0.119259 & 0.097621 & 0.152423 \\
\hline
\multirow{5}{*}{1000}
& MAE    & 0.024141 & 0.032439 & 0.036162 & 0.039273 & 0.038195 & 0.039322 & 0.038274 & 0.029353 & 0.044525 \\
& MSE    & 0.000877 & 0.001590 & 0.001982 & 0.002343 & 0.002228 & 0.002344 & 0.002231 & 0.001330 & 0.003116 \\
& Median & 0.020658 & 0.028489 & 0.031086 & 0.034594 & 0.034696 & 0.034421 & 0.034927 & 0.025598 & 0.037707 \\
& Q90    & 0.048351 & 0.064046 & 0.073582 & 0.079218 & 0.079125 & 0.079847 & 0.079148 & 0.059341 & 0.089249 \\
& Q95    & 0.057840 & 0.077597 & 0.085396 & 0.092410 & 0.090472 & 0.092422 & 0.089938 & 0.068942 & 0.108791 \\
\hline
\multirow{5}{*}{2000}
& MAE    & 0.018659 & 0.025467 & 0.028122 & 0.030252 & 0.029671 & 0.030232 & 0.029631 & 0.022884 & 0.035059 \\
& MSE    & 0.000534 & 0.000973 & 0.001192 & 0.001395 & 0.001326 & 0.001394 & 0.001323 & 0.000812 & 0.001860 \\
& Median & 0.016846 & 0.022130 & 0.024437 & 0.026106 & 0.026217 & 0.026044 & 0.026388 & 0.020445 & 0.029204 \\
& Q90    & 0.038356 & 0.052066 & 0.056480 & 0.061172 & 0.058984 & 0.061170 & 0.058737 & 0.046448 & 0.072581 \\
& Q95    & 0.044255 & 0.059942 & 0.066329 & 0.071238 & 0.069291 & 0.071240 & 0.068512 & 0.055781 & 0.082386 \\
\hline
\multirow{5}{*}{5000}
& MAE    & 0.011756 & 0.015152 & 0.016491 & 0.017618 & 0.017328 & 0.017611 & 0.017323 & 0.014490 & 0.020988 \\
& MSE    & 0.000215 & 0.000365 & 0.000439 & 0.000507 & 0.000490 & 0.000507 & 0.000489 & 0.000317 & 0.000683 \\
& Median & 0.010020 & 0.012336 & 0.013943 & 0.014923 & 0.014729 & 0.014922 & 0.014766 & 0.013010 & 0.017922 \\
& Q90    & 0.023354 & 0.032167 & 0.035496 & 0.037974 & 0.037047 & 0.037991 & 0.036445 & 0.027374 & 0.041582 \\
& Q95    & 0.027684 & 0.036030 & 0.040371 & 0.043431 & 0.043515 & 0.043434 & 0.042303 & 0.033452 & 0.049732 \\
\hline
\end{tabular}
\end{table*}

\FloatBarrier
\begin{table*}[h]
\centering
\caption{Results on DGP2 (Centered LogNormal): robust error summaries over $R=100$ replications.}
\label{tab:dgp2_lognormal}
\scriptsize
\setlength{\tabcolsep}{4pt}
\begin{tabular}{c|cccccccccc}
\hline
$n$ & Metric & l2 & Cauchy & GM & Adapt & AMB & GNC\_Adapt & GNC\_AMB & GNC\_TLS & ARE \\
\hline
\multirow{5}{*}{200}
& MAE    & 0.440732 & 2.383597 & 2.428866 & 2.437117 & 2.441007 & 2.457052 & 2.348575 & 2.228842 & 2.442037 \\
& MSE    & 0.319861 & 5.684898 & 5.905568 & 5.977074 & 5.964692 & 6.039979 & 5.522237 & 4.972983 & 5.968493 \\
& Median & 0.374759 & 2.385970 & 2.433266 & 2.458548 & 2.446208 & 2.460386 & 2.348119 & 2.225001 & 2.446614 \\
& Q90    & 0.835407 & 2.453506 & 2.501665 & 2.523363 & 2.513925 & 2.521825 & 2.452893 & 2.326485 & 2.530520 \\
& Q95    & 1.092987 & 2.474307 & 2.514922 & 2.537888 & 2.530098 & 2.537393 & 2.473458 & 2.343185 & 2.551792 \\
\hline
\multirow{5}{*}{500}
& MAE    & 0.284348 & 2.384501 & 2.431850 & 2.456983 & 2.446077 & 2.457022 & 2.342375 & 2.230460 & 2.456446 \\
& MSE    & 0.139719 & 5.686893 & 5.914873 & 6.037721 & 5.984414 & 6.037914 & 5.489974 & 4.976731 & 6.035749 \\
& Median & 0.233371 & 2.385287 & 2.431715 & 2.457309 & 2.446462 & 2.457377 & 2.335733 & 2.230171 & 2.458230 \\
& Q90    & 0.599346 & 2.426119 & 2.470482 & 2.496481 & 2.488803 & 2.496624 & 2.425090 & 2.284673 & 2.504184 \\
& Q95    & 0.708262 & 2.434107 & 2.482159 & 2.508471 & 2.499906 & 2.508461 & 2.444517 & 2.296450 & 2.517809 \\
\hline
\multirow{5}{*}{1000}
& MAE    & 0.211270 & 2.385299 & 2.432640 & 2.457747 & 2.448701 & 2.457779 & 2.329159 & 2.231928 & 2.457179 \\
& MSE    & 0.077294 & 5.690247 & 5.918278 & 6.041038 & 5.996701 & 6.041196 & 5.426657 & 4.982442 & 6.038572 \\
& Median & 0.172539 & 2.384034 & 2.432140 & 2.457033 & 2.448279 & 2.457028 & 2.324848 & 2.231803 & 2.457498 \\
& Q90    & 0.424567 & 2.417506 & 2.463839 & 2.489640 & 2.481998 & 2.489626 & 2.387608 & 2.270945 & 2.494162 \\
& Q95    & 0.530385 & 2.423898 & 2.470188 & 2.493565 & 2.486350 & 2.493547 & 2.406430 & 2.283688 & 2.506638 \\
\hline
\multirow{5}{*}{2000}
& MAE    & 0.156645 & 2.385716 & 2.433062 & 2.458169 & 2.449900 & 2.458201 & 2.318107 & 2.231655 & 2.458427 \\
& MSE    & 0.043052 & 5.691973 & 5.920085 & 6.042870 & 6.002305 & 6.043028 & 5.374561 & 4.980825 & 6.044297 \\
& Median & 0.121270 & 2.386073 & 2.433150 & 2.457977 & 2.449910 & 2.458020 & 2.313820 & 2.231244 & 2.459061 \\
& Q90    & 0.320521 & 2.409282 & 2.456259 & 2.480855 & 2.473205 & 2.480893 & 2.358750 & 2.261899 & 2.485424 \\
& Q95    & 0.394085 & 2.417107 & 2.461496 & 2.485910 & 2.478026 & 2.485956 & 2.375311 & 2.270079 & 2.491608 \\
\hline
\multirow{5}{*}{5000}
& MAE    & 0.105697 & 2.385195 & 2.432413 & 2.457464 & 2.449503 & 2.457496 & 2.307456 & 2.231414 & 2.458300 \\
& MSE    & 0.018055 & 5.689290 & 5.916752 & 6.039241 & 6.000178 & 6.039394 & 5.324652 & 4.979428 & 6.043399 \\
& Median & 0.091034 & 2.385240 & 2.432084 & 2.457331 & 2.449294 & 2.457379 & 2.306014 & 2.231975 & 2.458453 \\
& Q90    & 0.216724 & 2.399676 & 2.445936 & 2.470685 & 2.462998 & 2.470644 & 2.330204 & 2.251057 & 2.474467 \\
& Q95    & 0.262409 & 2.405306 & 2.450735 & 2.475080 & 2.466920 & 2.475065 & 2.338020 & 2.255622 & 2.477549 \\
\hline
\end{tabular}
\end{table*}

\FloatBarrier
\begin{table*}[h]
\centering
\caption{Results on DGP3 (Student-$t$): robust error summaries over $R=100$ replications.}
\label{tab:dgp3_studentt}
\scriptsize
\setlength{\tabcolsep}{4pt}
\begin{tabular}{c|cccccccccc}
\hline
$n$ & Metric & l2 & Cauchy & GM & Adapt & AMB & GNC\_Adapt & GNC\_AMB & GNC\_TLS & ARE \\
\hline
\multirow{5}{*}{200}
& MAE    & 0.079329 & 0.077027 & 0.085139 & 0.091804 & 0.089159 & 0.091806 & 0.088340 & 0.075760 & 0.106217 \\
& MSE    & 0.010140 & 0.009076 & 0.011105 & 0.013011 & 0.012191 & 0.013010 & 0.011964 & 0.009004 & 0.017967 \\
& Median & 0.065749 & 0.066426 & 0.074241 & 0.079843 & 0.077080 & 0.079817 & 0.076304 & 0.065225 & 0.086817 \\
& Q90    & 0.165889 & 0.156022 & 0.175618 & 0.185670 & 0.181116 & 0.185723 & 0.181239 & 0.155254 & 0.216114 \\
& Q95    & 0.197331 & 0.181987 & 0.201629 & 0.219082 & 0.212402 & 0.219154 & 0.211227 & 0.181042 & 0.265095 \\
\hline
\multirow{5}{*}{500}
& MAE    & 0.053192 & 0.050686 & 0.055188 & 0.059119 & 0.057425 & 0.059175 & 0.057249 & 0.050691 & 0.070144 \\
& MSE    & 0.004389 & 0.003959 & 0.004725 & 0.005423 & 0.005150 & 0.005426 & 0.005107 & 0.004040 & 0.007499 \\
& Median & 0.047555 & 0.042452 & 0.048275 & 0.049438 & 0.048677 & 0.049481 & 0.048193 & 0.041320 & 0.063075 \\
& Q90    & 0.108891 & 0.100059 & 0.111604 & 0.121137 & 0.118760 & 0.120968 & 0.118494 & 0.103940 & 0.144723 \\
& Q95    & 0.124714 & 0.122820 & 0.130709 & 0.143039 & 0.141928 & 0.143011 & 0.136721 & 0.123777 & 0.162166 \\
\hline
\multirow{5}{*}{1000}
& MAE    & 0.036055 & 0.034389 & 0.036687 & 0.038872 & 0.038216 & 0.038801 & 0.038277 & 0.035842 & 0.045915 \\
& MSE    & 0.002049 & 0.001818 & 0.002119 & 0.002405 & 0.002328 & 0.002397 & 0.002328 & 0.001978 & 0.003375 \\
& Median & 0.030219 & 0.029908 & 0.031209 & 0.032017 & 0.031530 & 0.031926 & 0.031413 & 0.030449 & 0.039262 \\
& Q90    & 0.076757 & 0.071994 & 0.077547 & 0.081623 & 0.079673 & 0.081599 & 0.080059 & 0.072346 & 0.096297 \\
& Q95    & 0.087859 & 0.085480 & 0.095375 & 0.100148 & 0.098976 & 0.100240 & 0.098355 & 0.083406 & 0.117685 \\
\hline
\multirow{5}{*}{2000}
& MAE    & 0.027005 & 0.025412 & 0.027685 & 0.029742 & 0.029301 & 0.029659 & 0.029244 & 0.026076 & 0.035955 \\
& MSE    & 0.001130 & 0.001019 & 0.001203 & 0.001371 & 0.001331 & 0.001361 & 0.001327 & 0.001082 & 0.002001 \\
& Median & 0.024092 & 0.021646 & 0.023765 & 0.025259 & 0.025824 & 0.025264 & 0.025472 & 0.022653 & 0.029967 \\
& Q90    & 0.053873 & 0.050953 & 0.057300 & 0.062042 & 0.060372 & 0.061152 & 0.059646 & 0.054660 & 0.073035 \\
& Q95    & 0.063653 & 0.059921 & 0.065569 & 0.071133 & 0.069951 & 0.070878 & 0.070029 & 0.063390 & 0.083720 \\
\hline
\multirow{5}{*}{5000}
& MAE    & 0.016767 & 0.015550 & 0.016721 & 0.017673 & 0.017497 & 0.017613 & 0.017462 & 0.016209 & 0.022485 \\
& MSE    & 0.000438 & 0.000380 & 0.000446 & 0.000505 & 0.000494 & 0.000503 & 0.000493 & 0.000417 & 0.000775 \\
& Median & 0.014579 & 0.013322 & 0.014540 & 0.015093 & 0.015065 & 0.015025 & 0.015005 & 0.013681 & 0.019660 \\
& Q90    & 0.033488 & 0.031232 & 0.033738 & 0.035248 & 0.035180 & 0.035245 & 0.035151 & 0.033712 & 0.045893 \\
& Q95    & 0.041146 & 0.037998 & 0.041529 & 0.044344 & 0.043462 & 0.044339 & 0.043444 & 0.041594 & 0.051865 \\
\hline
\end{tabular}
\end{table*}

\FloatBarrier
\begin{table*}[h]
\centering
\caption{Results on DGP4 (Symmetric Pareto): robust error summaries over $R=100$ replications.}
\label{tab:dgp4_pareto}
\scriptsize
\setlength{\tabcolsep}{4pt}
\begin{tabular}{c|cccccccccc}
\hline
$n$ & Metric & l2 & Cauchy & GM & Adapt & AMB & GNC\_Adapt & GNC\_AMB & GNC\_TLS & ARE \\
\hline
\multirow{5}{*}{200}
& MAE    & 0.412362 & 1.438328 & 1.518497 & 1.522176 & 1.517095 & 1.500790 & 1.564941 & 1.595852 & 1.326587 \\
& MSE    & 0.590463 & 2.072604 & 2.413169 & 2.459166 & 2.413946 & 2.255401 & 2.458244 & 2.645237 & 1.767638 \\
& Median & 0.253832 & 1.436787 & 1.502406 & 1.498419 & 1.502898 & 1.501094 & 1.564170 & 1.650789 & 1.325492 \\
& Q90    & 0.848063 & 1.514769 & 1.572911 & 1.565906 & 1.574786 & 1.566497 & 1.664394 & 1.763932 & 1.438845 \\
& Q95    & 1.155875 & 1.537037 & 1.602260 & 1.596185 & 1.600452 & 1.599216 & 1.694309 & 1.795882 & 1.461990 \\
\hline
\multirow{5}{*}{500}
& MAE    & 0.346482 & 1.439051 & 1.505764 & 1.513167 & 1.516621 & 1.497524 & 1.574814 & 1.638352 & 1.339161 \\
& MSE    & 0.441919 & 2.072407 & 2.268632 & 2.356217 & 2.371738 & 2.243820 & 2.482337 & 2.718181 & 1.798361 \\
& Median & 0.214237 & 1.435626 & 1.503844 & 1.501487 & 1.503984 & 1.495642 & 1.576356 & 1.655792 & 1.345156 \\
& Q90    & 0.679670 & 1.488627 & 1.554075 & 1.548618 & 1.551994 & 1.545145 & 1.633025 & 1.723029 & 1.425720 \\
& Q95    & 1.001716 & 1.505555 & 1.565078 & 1.559460 & 1.565327 & 1.556368 & 1.648180 & 1.741259 & 1.447595 \\
\hline
\multirow{5}{*}{1000}
& MAE    & 0.404736 & 1.435188 & 1.568055 & 1.499761 & 1.502030 & 1.496583 & 1.581408 & 1.650191 & 1.172465 \\
& MSE    & 6.035140 & 2.060631 & 4.486813 & 2.249927 & 2.260321 & 2.244539 & 2.502444 & 2.742182 & 1.592772 \\
& Median & 0.174446 & 1.433928 & 1.502667 & 1.497949 & 1.503413 & 1.497246 & 1.585403 & 1.658140 & 1.350682 \\
& Q90    & 0.617388 & 1.473958 & 1.538876 & 1.532762 & 1.537795 & 1.533496 & 1.626974 & 1.709116 & 1.419330 \\
& Q95    & 0.977243 & 1.485230 & 1.548406 & 1.543379 & 1.549137 & 1.545707 & 1.641047 & 1.721744 & 1.435046 \\
\hline
\multirow{5}{*}{2000}
& MAE    & 0.277339 & 1.483650 & 1.553332 & 1.549032 & 1.553667 & 1.498893 & 1.583741 & 1.640998 & 1.358529 \\
& MSE    & 1.570783 & 3.414067 & 3.633623 & 3.623838 & 3.619150 & 2.247003 & 2.509354 & 2.721683 & 1.850885 \\
& Median & 0.134661 & 1.434549 & 1.503553 & 1.499047 & 1.504036 & 1.497680 & 1.589020 & 1.660958 & 1.364235 \\
& Q90    & 0.499724 & 1.461581 & 1.527771 & 1.521979 & 1.528136 & 1.521673 & 1.621502 & 1.693832 & 1.421070 \\
& Q95    & 0.723037 & 1.468161 & 1.534138 & 1.527576 & 1.533661 & 1.527547 & 1.628375 & 1.702626 & 1.432571 \\
\hline
\multirow{5}{*}{5000}
& MAE    & 0.172371 & 1.433381 & 1.503061 & 1.498617 & 1.504033 & 1.498603 & 1.585251 & 1.610063 & 1.376493 \\
& MSE    & 0.291315 & 2.054726 & 2.259312 & 2.245967 & 2.262231 & 2.245924 & 2.518299 & 2.668343 & 1.896178 \\
& Median & 0.095233 & 1.433980 & 1.502981 & 1.498358 & 1.503860 & 1.498899 & 1.592539 & 1.660890 & 1.377615 \\
& Q90    & 0.312437 & 1.447404 & 1.516828 & 1.511865 & 1.517150 & 1.511360 & 1.608947 & 1.678992 & 1.423776 \\
& Q95    & 0.446499 & 1.452532 & 1.519658 & 1.515645 & 1.521086 & 1.515568 & 1.613800 & 1.682779 & 1.436759 \\
\hline
\end{tabular}
\end{table*}

\FloatBarrier
\begin{table*}[h]
\centering
\caption{Results on DGP5 (Random contamination): robust error summaries over $R=100$ replications.}
\label{tab:dgp5_contamination_random}
\scriptsize
\setlength{\tabcolsep}{4pt}
\begin{tabular}{c|cccccccccc}
\toprule
$n$ & Metric & l2 & Cauchy & GM & Adapt & AMB & GNC\_Adapt & GNC\_AMB & GNC\_TLS & ARE \\
\hline
\multirow{5}{*}{200}
& MAE    & 1.275762 & 0.076195 & 0.085042 & 0.092851 & 0.088960 & 0.092902 & 0.082093 & 0.069203 & 0.101623 \\
& MSE    & 2.589767 & 0.009235 & 0.011385 & 0.013490 & 0.012538 & 0.013503 & 0.010758 & 0.007540 & 0.016209 \\
& Median & 1.034393 & 0.063265 & 0.071296 & 0.077451 & 0.074361 & 0.077476 & 0.064493 & 0.058816 & 0.085344 \\
& Q90    & 2.562352 & 0.154705 & 0.171824 & 0.187716 & 0.182085 & 0.187708 & 0.170956 & 0.143705 & 0.219979 \\
& Q95    & 3.040650 & 0.198719 & 0.218241 & 0.232770 & 0.224837 & 0.232767 & 0.207915 & 0.179176 & 0.251584 \\
\hline
\multirow{5}{*}{500}
& MAE    & 0.803141 & 0.045911 & 0.051108 & 0.055784 & 0.054351 & 0.055812 & 0.051871 & 0.043391 & 0.065892 \\
& MSE    & 0.999150 & 0.003058 & 0.003778 & 0.004492 & 0.004241 & 0.004495 & 0.003863 & 0.002878 & 0.006678 \\
& Median & 0.647632 & 0.042344 & 0.047195 & 0.050444 & 0.050155 & 0.050443 & 0.046949 & 0.037627 & 0.055327 \\
& Q90    & 1.648696 & 0.088117 & 0.097799 & 0.107246 & 0.102712 & 0.107400 & 0.100482 & 0.087705 & 0.133173 \\
& Q95    & 1.987630 & 0.103240 & 0.113667 & 0.125033 & 0.119585 & 0.124933 & 0.113398 & 0.103239 & 0.162270 \\
\hline
\multirow{5}{*}{1000}
& MAE    & 0.559301 & 0.033795 & 0.037628 & 0.040745 & 0.039726 & 0.040750 & 0.038179 & 0.030827 & 0.047724 \\
& MSE    & 0.478545 & 0.001704 & 0.002122 & 0.002517 & 0.002390 & 0.002517 & 0.002215 & 0.001423 & 0.003510 \\
& Median & 0.493729 & 0.028837 & 0.032411 & 0.035934 & 0.034792 & 0.035898 & 0.033473 & 0.027246 & 0.038480 \\
& Q90    & 1.097783 & 0.068452 & 0.077436 & 0.084156 & 0.082611 & 0.084147 & 0.076178 & 0.059305 & 0.097886 \\
& Q95    & 1.313985 & 0.080211 & 0.088760 & 0.097425 & 0.095241 & 0.097411 & 0.092335 & 0.073964 & 0.115841 \\
\hline
\multirow{5}{*}{2000}
& MAE    & 0.402444 & 0.025673 & 0.028418 & 0.030721 & 0.030015 & 0.030728 & 0.029719 & 0.023478 & 0.036463 \\
& MSE    & 0.252472 & 0.000999 & 0.001222 & 0.001433 & 0.001362 & 0.001434 & 0.001334 & 0.000846 & 0.002092 \\
& Median & 0.342546 & 0.022232 & 0.024743 & 0.026295 & 0.026008 & 0.026305 & 0.026590 & 0.020136 & 0.030277 \\
& Q90    & 0.814902 & 0.052501 & 0.057734 & 0.063119 & 0.060448 & 0.063057 & 0.059638 & 0.046563 & 0.077859 \\
& Q95    & 0.983756 & 0.063525 & 0.067954 & 0.074476 & 0.071649 & 0.074469 & 0.071677 & 0.054432 & 0.091296 \\
\hline
\multirow{5}{*}{5000}
& MAE    & 0.254088 & 0.015730 & 0.017206 & 0.018488 & 0.018124 & 0.018487 & 0.018104 & 0.014815 & 0.023329 \\
& MSE    & 0.099986 & 0.000389 & 0.000470 & 0.000546 & 0.000526 & 0.000546 & 0.000520 & 0.000330 & 0.000835 \\
& Median & 0.214689 & 0.013191 & 0.014152 & 0.015257 & 0.014829 & 0.015312 & 0.014979 & 0.013506 & 0.020551 \\
& Q90    & 0.528963 & 0.031760 & 0.035809 & 0.038966 & 0.038035 & 0.038986 & 0.038560 & 0.027583 & 0.047545 \\
& Q95    & 0.641245 & 0.037724 & 0.041986 & 0.045497 & 0.044011 & 0.045536 & 0.043942 & 0.032985 & 0.055688 \\
\hline
\end{tabular}
\end{table*}

\FloatBarrier
\begin{table*}[H]
\centering
\caption{Results on DGP5 (Adversarial contamination): robust error summaries over $R=100$ replications.}
\label{tab:dgp5_contamination_adversarial}
\scriptsize
\setlength{\tabcolsep}{4pt}
\begin{tabular}{c|cccccccccc}
\hline
$n$ & Metric & l2 & Cauchy & GM & Adapt & AMB & GNC\_Adapt & GNC\_AMB & GNC\_TLS & ARE \\
\hline
\multirow{5}{*}{200}
& MAE    & 4.998228 & 0.076277 & 0.085139 & 0.093145 & 0.089440 & 0.093207 & 0.072658 & 0.069398 & 0.103908 \\
& MSE    & 24.987188 & 0.009264 & 0.011434 & 0.013583 & 0.012628 & 0.013600 & 0.008495 & 0.007742 & 0.016552 \\
& Median & 4.998399 & 0.063285 & 0.072447 & 0.080556 & 0.077372 & 0.080661 & 0.059623 & 0.057831 & 0.091392 \\
& Q90    & 5.084834 & 0.158220 & 0.175776 & 0.190045 & 0.185091 & 0.190049 & 0.154349 & 0.143999 & 0.212980 \\
& Q95    & 5.108672 & 0.194569 & 0.210418 & 0.232594 & 0.216446 & 0.232717 & 0.181021 & 0.173504 & 0.255916 \\
\hline
\multirow{5}{*}{500}
& MAE    & 4.999549 & 0.045390 & 0.050404 & 0.054917 & 0.053448 & 0.054955 & 0.044563 & 0.042151 & 0.067477 \\
& MSE    & 24.997235 & 0.003090 & 0.003836 & 0.004572 & 0.004302 & 0.004578 & 0.003019 & 0.002751 & 0.007088 \\
& Median & 4.999719 & 0.040502 & 0.044703 & 0.046358 & 0.046212 & 0.046366 & 0.037025 & 0.036002 & 0.055582 \\
& Q90    & 5.050312 & 0.089395 & 0.100666 & 0.109910 & 0.105790 & 0.110204 & 0.091927 & 0.086210 & 0.142707 \\
& Q95    & 5.064844 & 0.105286 & 0.118548 & 0.130789 & 0.127427 & 0.130905 & 0.106573 & 0.102324 & 0.159420 \\
\hline
\multirow{5}{*}{1000}
& MAE    & 5.001430 & 0.033238 & 0.037159 & 0.040485 & 0.039282 & 0.040500 & 0.032143 & 0.030005 & 0.047662 \\
& MSE    & 25.015113 & 0.001671 & 0.002089 & 0.002487 & 0.002350 & 0.002489 & 0.001573 & 0.001385 & 0.003576 \\
& Median & 5.001633 & 0.028801 & 0.031912 & 0.034680 & 0.034012 & 0.034688 & 0.028010 & 0.027005 & 0.039794 \\
& Q90    & 5.037464 & 0.067114 & 0.075948 & 0.083690 & 0.080102 & 0.083690 & 0.064081 & 0.060579 & 0.096793 \\
& Q95    & 5.044215 & 0.077899 & 0.088521 & 0.096025 & 0.093886 & 0.096265 & 0.075479 & 0.071472 & 0.119724 \\
\hline
\multirow{5}{*}{2000}
& MAE    & 5.001119 & 0.025654 & 0.028436 & 0.030799 & 0.030083 & 0.030811 & 0.024935 & 0.023116 & 0.036394 \\
& MSE    & 25.011676 & 0.001013 & 0.001251 & 0.001476 & 0.001393 & 0.001478 & 0.000961 & 0.000838 & 0.002057 \\
& Median & 5.002635 & 0.020993 & 0.023428 & 0.025935 & 0.025570 & 0.025954 & 0.021426 & 0.019531 & 0.032832 \\
& Q90    & 5.029097 & 0.052707 & 0.058848 & 0.063274 & 0.060947 & 0.063288 & 0.052107 & 0.050658 & 0.073951 \\
& Q95    & 5.036158 & 0.063923 & 0.069545 & 0.075759 & 0.074667 & 0.075907 & 0.059192 & 0.057579 & 0.086389 \\
\hline
\multirow{5}{*}{5000}
& MAE    & 5.001120 & 0.015621 & 0.016972 & 0.018272 & 0.017900 & 0.018279 & 0.015694 & 0.014902 & 0.021548 \\
& MSE    & 25.011411 & 0.000387 & 0.000465 & 0.000539 & 0.000520 & 0.000539 & 0.000384 & 0.000336 & 0.000751 \\
& Median & 5.001860 & 0.013595 & 0.014565 & 0.015719 & 0.015379 & 0.015783 & 0.013764 & 0.013436 & 0.017448 \\
& Q90    & 5.019530 & 0.031161 & 0.034034 & 0.036961 & 0.035849 & 0.036985 & 0.030413 & 0.028134 & 0.044297 \\
& Q95    & 5.022983 & 0.037216 & 0.040165 & 0.044437 & 0.043214 & 0.044466 & 0.037985 & 0.035315 & 0.054321 \\
\hline
\end{tabular}
\end{table*}

Across the five data-generating processes (DGP1--DGP5), the dominant pattern is that \texttt{ARE} offers its clearest benefits precisely in the regimes where robustness is genuinely needed: skewed heavy tails, infinite-variance tails, and explicit contamination. In the skewed heavy-tailed DGP2 (centered log-normal) and the infinite-variance DGP4 (symmetric Pareto with tail index $\xi\in(1,2)$), several non-MoM robust losses (e.g., \texttt{Cauchy/GM/Adapt/AMB} and their GNC counterparts) exhibit a pronounced non-vanishing error floor: their MAE and Median remain essentially flat as $n$ grows, with high quantiles (Q90/Q95) also failing to improve materially. In contrast, \texttt{ARE} consistently suppresses the tail risk in these settings, yielding markedly smaller Q90/Q95 and, in DGP4, substantially improved overall error summaries at representative sample sizes (e.g., a striking reduction in MAE/MSE at $n=1000$). This behavior is consistent with the two-layer robustness of the method: the inner GNC mechanism stabilizes the optimization of a nonconvex robust objective, while the outer median-of-means aggregation truncates the influence of a minority of blocks containing extreme observations, thereby preventing the estimator from being dominated by rare but catastrophic samples.

The advantage becomes even more transparent under explicit contamination (DGP5). Under random contamination, the sample mean (\texttt{l2}) suffers a dramatic degradation in all metrics, while \texttt{ARE} maintains small errors and noticeably better tail quantiles, indicating strong protection against sporadic outliers. Under adversarial contamination, \texttt{l2} is effectively ``locked'' by the injected outliers: its MAE stays near $5$ and does not meaningfully decrease with $n$, reflecting a complete breakdown of classical averaging. In sharp contrast, \texttt{ARE} remains accurate and continues to improve with $n$, with particularly strong control of Q90/Q95, which highlights its robustness to worst-case corruption consistent with MoM-type guarantees (so long as a majority of blocks remain uncontaminated). Importantly, in the benign baselines DGP1 (Gaussian) and DGP3 (Student-$t$ with finite variance), \texttt{ARE} incurs only a modest constant-factor efficiency loss relative to the best light-tail methods (e.g., \texttt{l2} or \texttt{GNC\_TLS}) and does not exhibit instability. Overall, these results suggest that \texttt{ARE} trades a small amount of efficiency in easy regimes for substantial and reliable gains in heavy-tailed and contaminated regimes, with its most consistent improvements appearing in the high-quantile error metrics (Q90/Q95) that directly capture tail-risk robustness.

\section{Introduction to Adaptive Losses.} \label{IntroductionAdaptive}
In this paper, we employ the adaptive loss function introduced by \citet{BarronCVPR2019}, defined as follows:
\begin{equation}\label{adaptiveloss}
\rho(\epsilon; \alpha, c)
= \frac{|\alpha-2|}{\alpha}\left(\left(1+\frac{(\epsilon/c)^2}{|\alpha-2|}\right)^{\alpha/2}-1\right),
\end{equation}
where $\alpha \in (-\infty, 2]$ is a shape parameter controlling the loss's robustness, while $c > 0$ is a scale parameter governing the quadratic bowl near $\epsilon = 0$. Through simple analysis, we can immediately recover several common losses---such as the L2, Charbonnier/pseudo-Huber/L1-L2, Cauchy/Lorentzian \citep{black1996robust}, Geman--McClure \citep{geman1986bayesian}, and Welsch/Leclerc losses \citep{dennis1978techniques,leclerc1989constructing}---as special cases of the adaptive loss $\rho(\epsilon;\alpha,c)$ in \eqref{adaptiveloss} by setting $\alpha$ to particular values, as summarized as follows,
\[
\rho(\epsilon; \alpha, c) =
\begin{cases}
\frac{1}{2} (\epsilon/c)^2 & \alpha = 2 \\
\sqrt{(\epsilon/c)^2 + 1} - 1 & \alpha = 1 \\
\log \bigl( \frac{1}{2} (\epsilon/c)^2 + 1 \bigr) & \alpha = 0 \\
2 (\epsilon/c)^2 / \bigl( (\epsilon/c)^2 + 4 \bigr) & \alpha = -2 \\
1 - \exp \bigl( -\frac{1}{2} (\epsilon/c)^2 \bigr) & \alpha = -\infty
\end{cases}, \quad \frac{\partial \rho}{\partial \epsilon}(\epsilon; \alpha, c) =
\begin{cases}
\epsilon / c^2 & \alpha = 2 \\
(\epsilon / c^2) / \sqrt{(\epsilon/c)^2 + 1} & \alpha = 1 \\
2\epsilon / (\epsilon^2 + 2 c^2) & \alpha = 0 \\
(\epsilon / c^2) \bigl( (\epsilon/c)^2 / 4 + 1 \bigr)^{-3/2} & \alpha = -2 \\
(\epsilon / c^2) \exp \bigl( -\frac{1}{2} (\epsilon/c)^2 \bigr) & \alpha = -\infty
\end{cases}
\]

\begin{figure}[H]
\centering 
\begin{minipage}[t]{0.98\textwidth}
\centering
\includegraphics[scale=0.38]{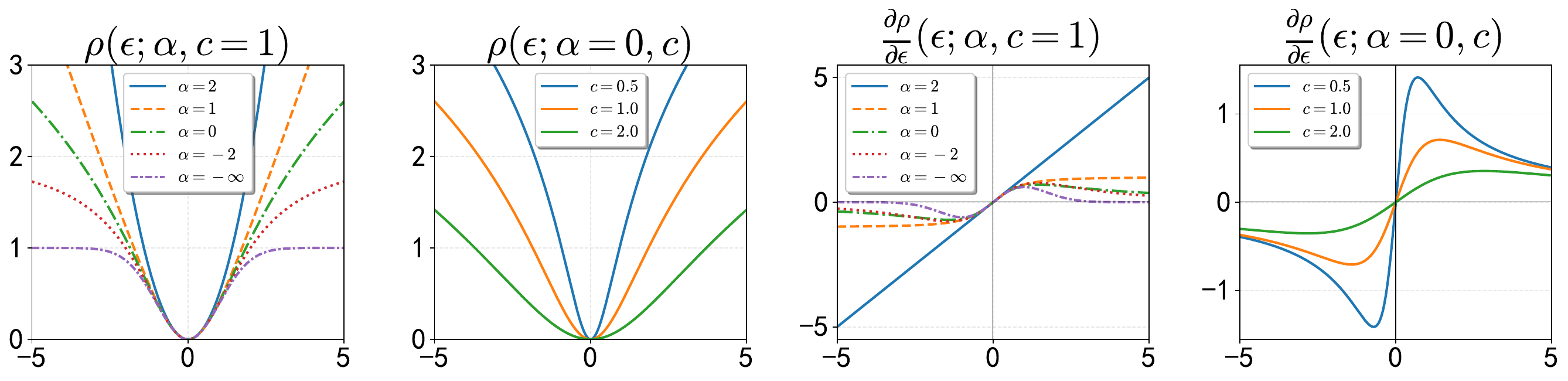}
\end{minipage}

\begin{minipage}[t]{0.98\textwidth}
\centering
\includegraphics[scale=0.38]{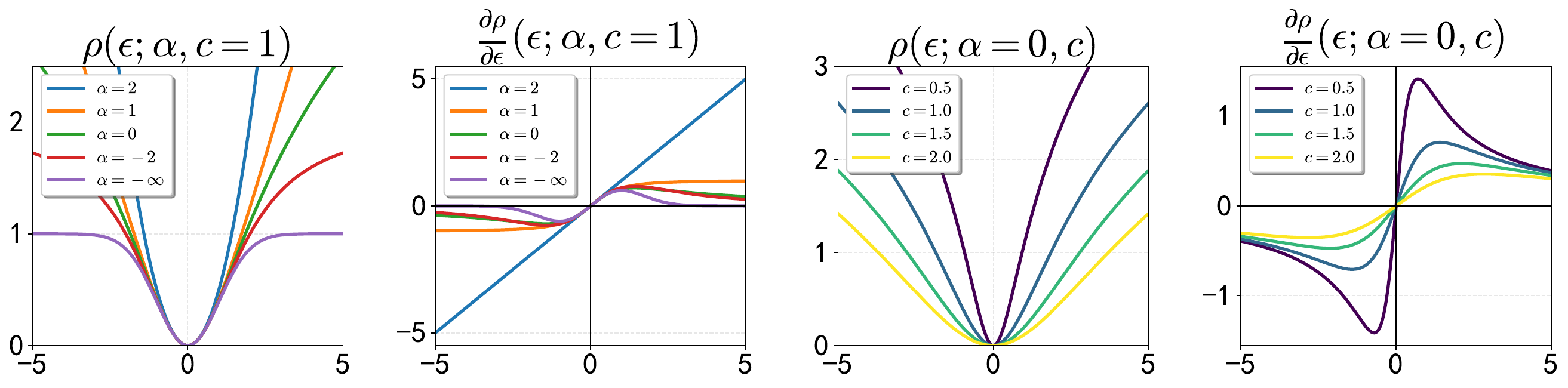}
\end{minipage}
\caption{\textbf{Adaptive robust loss $\rho(\epsilon;\alpha,c)$ and its gradient under different parameter settings.}
From left to right and top to bottom, the four panels show:
(\emph{top-left}) the loss $\rho(\epsilon;\alpha,c)$ for varying shape $\alpha$ with fixed $c=1$;
(\emph{top-right}) the corresponding gradient $\partial\rho/\partial\epsilon$ for varying $\alpha$ with fixed $c=1$;
(\emph{bottom-left}) the loss for varying scale $c$ with fixed $\alpha=0$;
(\emph{bottom-right}) the corresponding gradient for varying $c$ with fixed $\alpha=0$.}
\label{adaptivellos} 
\end{figure}


The adaptive loss $\rho(\epsilon;\alpha,c)$ provides a unified two-parameter family that spans a broad spectrum of robust penalties through a shape parameter $\alpha\in(-\infty,2]$ and a scale parameter $c>0$.
As illustrated in Figure \ref{adaptivellos}, decreasing $\alpha$ continuously interpolates from the quadratic $\ell_2$ loss ($\alpha=2$), which assigns unbounded influence to large residuals, to increasingly robust losses whose growth slows down for large $|\epsilon|$.
This continuum makes $\alpha$ a direct knob for robustness: smaller $\alpha$ reduces the marginal penalty on large errors and thus mitigates the effect of outliers or heavy-tailed noise. The scale parameter $c$ controls the width of the near-zero quadratic regime.
For fixed $\alpha$ (e.g., $\alpha=0$), larger $c$ enlarges the approximately quadratic region around $\epsilon=0$, thereby behaving more like least squares on moderate residuals and improving statistical efficiency under mild noise.
Conversely, smaller $c$ triggers earlier deviation from quadratic growth, increasing robustness by treating moderate-to-large residuals more conservatively.
Hence, $c$ mediates a local efficiency--robustness trade-off by determining the transition point at which the loss begins to downweight large errors.

The gradients $\partial\rho/\partial\epsilon$ make the robustness mechanism explicit through the induced influence function.
When $\alpha$ is small, the gradient decays rapidly as $|\epsilon|$ grows and can approach zero for very large residuals, effectively suppressing the contribution of extreme observations during optimization.
This bounded (or rapidly vanishing) influence stabilizes gradient-based updates and prevents optimization from being dominated by a small number of corrupted samples.
The scale $c$ further adjusts the decay rate of $\partial\rho/\partial\epsilon$: smaller $c$ yields faster saturation (stronger outlier rejection but potentially slower progress when many residuals are moderately large), whereas larger $c$ preserves larger gradients over a wider range (often faster convergence locally but weaker robustness).
Overall, the $(\alpha,c)$ parameterization enables data-dependent control of robustness versus efficiency within a single coherent framework, which is particularly useful for robust estimation and learning under heavy-tailed or contaminated data.

\section{The GNC Paradigm for Mean Estimation in Subsection~\ref{Estimationmean}} \label{realstep1solve}

This appendix details the graduated nonconvexity (GNC) procedure used to solve the nonconvex subproblem \eqref{Xstep} (equivalently \eqref{dualityweight}) after obtaining $(\alpha^{(1)},c^{(1)})$.
We present (i) a GNC surrogate for the adaptive loss, (ii) the induced IRLS weights and the associated Black--Rangarajan (BR) regularization, and (iii) a complete GNC-IRLS algorithm with practical continuation schedules and stopping criteria.

Let $f(\beta,\alpha^{(1)})$ be a shape mapping controlled by a continuation parameter $\beta$.
We construct a one-parameter family of surrogate losses by replacing the shape parameter $\alpha^{(1)}$ with $f(\beta,\alpha^{(1)})$:
\begin{equation}
\label{eq:app_surrogate_def}
\rho_{\beta}(\epsilon;\alpha^{(1)},c^{(1)})
\;\coloneqq\;
\rho\!\left(\epsilon;\ f(\beta,\alpha^{(1)}),\ c^{(1)}\right),
\end{equation}
where $\rho(\cdot;\alpha,c)$ denotes the adaptive robust loss used in the main text. To avoid ambiguity and ensure numerical stability, we recommend writing $\rho_\beta$ in the standard Barron form with a scale parameter:
\begin{equation}
\label{eq:app_barron_scaled}
\rho_{\beta}(\epsilon;\alpha^{(1)},c^{(1)})
=
\frac{|f-2|}{f}
\left[
\left(
\frac{(\epsilon/c^{(1)})^{2}}{|f-2|}+1
\right)^{\frac{f}{2}}
-1
\right],
\qquad f \equiv f(\beta,\alpha^{(1)}).
\end{equation}
This definition guarantees the correct quadratic limit as $f\to 2$ and matches the commonly used scaling convention (cf.\ Barron-style adaptive losses).
If a different scaling convention is adopted in the main text, \eqref{eq:app_barron_scaled} should be adjusted accordingly; the derivations below remain identical up to constant factors.

A valid shape mapping must satisfy
\begin{equation}
\label{eq:app_f_conditions}
\lim_{\beta\to \beta_1} f(\beta,\alpha^{(1)}) = 2,
\qquad
\lim_{\beta\to \beta_2} f(\beta,\alpha^{(1)}) = \alpha^{(1)},
\qquad
f(\beta,\alpha^{(1)}) \le 2,
\end{equation}
where $\beta_1,\beta_2$ are surrogate-dependent constants (or endpoints) determined by the chosen parameterization of $\beta$.
Intuitively, $f(\beta,\alpha^{(1)})\to 2$ yields a convex quadratic surrogate, while $f(\beta,\alpha^{(1)})\to \alpha^{(1)}$ recovers the original adaptive loss (typically nonconvex). We consider three concrete choices of $f(\beta,\alpha^{(1)})$:

\noindent\textbf{Example 1 (Polynomial decay; $\beta:\ \infty\to 1$).}
\begin{equation}
\label{eq:app_f1}
f_1(\beta,\alpha^{(1)})=2-\frac{2-\alpha^{(1)}}{\beta^{p}},\qquad p>0,
\end{equation}
so that $\beta\to\infty$ implies $f_1\to 2$ and $\beta\to 1$ implies $f_1\to \alpha^{(1)}$.

\noindent\textbf{Example 2 (Exponential transition; $\beta:\ 0^+\to \infty$).}
\begin{equation}
\label{eq:app_f2}
f_2(\beta,\alpha^{(1)})
=
\alpha^{(1)} e^{-1/\beta} + 2\bigl(1-e^{-1/\beta}\bigr),
\end{equation}
so that $\beta\to 0^+$ implies $f_2\to 2$ and $\beta\to\infty$ implies $f_2\to \alpha^{(1)}$.

\noindent\textbf{Example 3 (Rational form; $\beta:\ 0^+\to \infty$).}
\begin{equation}
\label{eq:app_f3}
f_3(\beta,\alpha^{(1)})=\frac{2+\alpha^{(1)}\beta^{q}}{1+\beta^{q}},\qquad q>0,
\end{equation}
so that $\beta\to 0^+$ implies $f_3\to 2$ and $\beta\to\infty$ implies $f_3\to \alpha^{(1)}$.

\begin{figure}[H]
\centering
\includegraphics[scale=0.38]{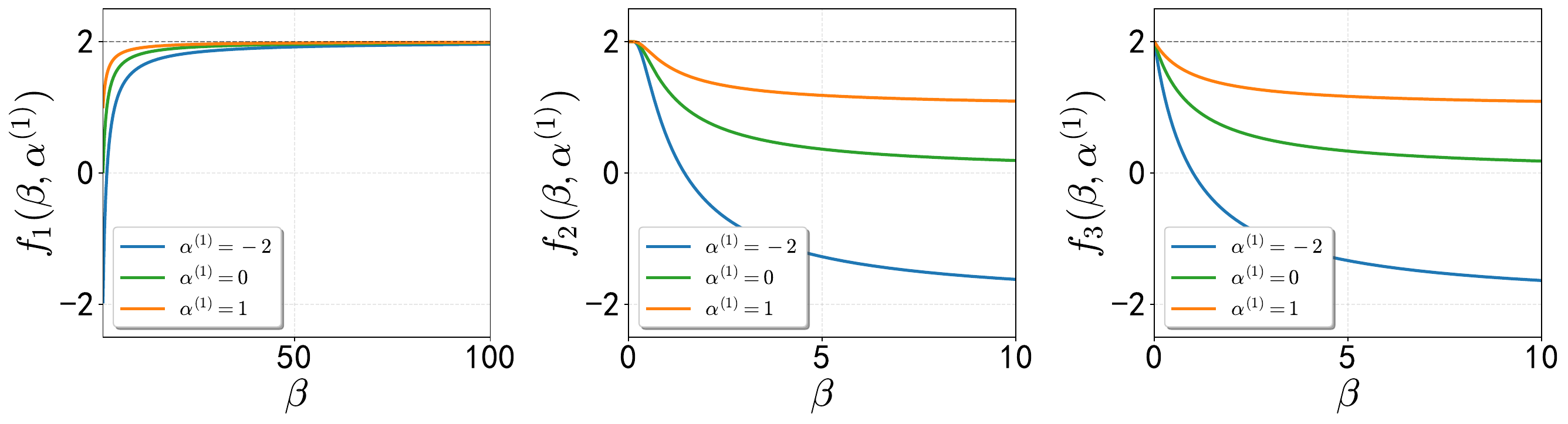}
\caption{Three instantiations of the interpolation function $f(\beta,\alpha^{(1)})$: (a) polynomial decay $f_1 = 2 - \frac{2-\alpha^{(1)}}{\beta^{p}}$ with $\beta: \infty \to 1$; (b) exponential transition $f_2 = \alpha^{(1)} e^{-1/\beta} + 2(1-e^{-1/\beta})$ with $\beta: 0^+ \to \infty$; (c) rational form $f_3 = \frac{2 + \alpha^{(1)} \beta^q}{1+\beta^q}$ with $\beta: 0^+ \to \infty$, all smoothly interpolating between the quadratic regime ($f=2$) and the adaptive regime ($f=\alpha^{(1)}$) for different values of $\alpha^{(1)} \in \{-2, 0, 1\}$.}\label{instantiationsexample} 
\end{figure}

Let $\epsilon_i=X_i-X$ and consider the surrogate objective
\begin{equation}
\label{eq:app_surrogate_obj}
\min_{X\in\mathbb{R}}\ \frac{1}{m}\sum_{i=1}^m \rho_{\beta}(\epsilon_i;\alpha^{(1)},c^{(1)}).
\end{equation}
At fixed $(\beta,\alpha^{(1)},c^{(1)})$, IRLS solves a sequence of weighted least-squares problems
\begin{equation}
\label{eq:app_wls}
\min_{X\in\mathbb{R}}\ \frac{1}{2m}\sum_{i=1}^m w_i\,\epsilon_i^2,
\end{equation}
with weights obtained by matching derivatives w.r.t.\ $X$:
\begin{equation}
\label{eq:app_weights_general}
w_i
\;=\;
\frac{1}{\epsilon_i}\frac{\partial}{\partial \epsilon_i}\rho_{\beta}(\epsilon_i;\alpha^{(1)},c^{(1)}).
\end{equation}

With $\rho_\beta$ defined in \eqref{eq:app_barron_scaled}, straightforward differentiation yields
\begin{equation}
\label{eq:app_weights_barron}
w_i
=
\frac{1}{(c^{(1)})^2}
\left(
\frac{(\epsilon_i/c^{(1)})^{2}}{|f-2|}+1
\right)^{\frac{f}{2}-1},
\qquad f=f(\beta,\alpha^{(1)}).
\end{equation}
For numerical robustness, it is often convenient to use the following special-case expressions when $f$ is close to singular values (and otherwise use \eqref{eq:app_weights_barron}):
\begin{equation}
\label{eq:app_weights_piecewise}
w_i=
\begin{cases}
\frac{1}{(c^{(1)})^2}, & f=2,\\[4pt]
\frac{2}{\epsilon_i^2 + 2(c^{(1)})^2}, & f=0,\\[6pt]
\frac{1}{(c^{(1)})^2}\left(\frac{(\epsilon_i/c^{(1)})^2}{|f-2|}+1\right)^{\frac{f}{2}-1}, & \text{otherwise},
\end{cases}
\end{equation}
and to clip $w_i$ to $(0,1]$ when required by the BR-duality constraint.

Optimizing \eqref{eq:app_wls} jointly over $(X,(w_i)^m_{i=1})$ without regularization degenerates to $w_i=0$ for all $i$.
To prevent this, we exploit the Black--Rangarajan duality~\citep{10.1007/BF00131148}.
Under standard conditions (e.g., concavity of $\phi(z)\!\coloneqq\!\rho_\beta(c^{(1)}\sqrt{z})$ in $z$), there exists an outlier-process penalty $\Phi_{\rho_\beta}(\cdot)$ such that
\begin{equation}
\label{eq:app_brd_identity}
\rho_{\beta}(\epsilon)=\min_{w\in(0,1]}\frac{1}{2}w\epsilon^2+\Phi_{\rho_{\beta}}(w).
\end{equation}
Plugging \eqref{eq:app_brd_identity} into \eqref{eq:app_surrogate_obj} yields the BR-regularized formulation
\begin{equation}
\label{eq:app_dual_problem}
\min_{X\in\mathbb{R},\, w_i\in(0,1]}\ \frac{1}{2m}\sum_{i=1}^m\Bigl(w_i\epsilon_i^2+\Phi_{\rho_{\beta}}(w_i)\Bigr),
\end{equation}
which matches \eqref{dualityweight} with the surrogate loss $\rho_\beta$.
At fixed $X$, minimizing \eqref{eq:app_dual_problem} over $w_i$ recovers the IRLS update \eqref{eq:app_weights_general}. For fixed weights, \eqref{eq:app_wls} admits a closed-form minimizer:
\begin{equation}
\label{eq:app_x_update}
X
\leftarrow
\frac{\sum_{i=1}^m w_i X_i}{\sum_{i=1}^m w_i}.
\end{equation}

The GNC continuation progressively changes $\beta$ so that $f(\beta,\alpha^{(1)})$ moves from the quadratic regime ($f\approx 2$) to the target regime ($f\approx \alpha^{(1)}$).

Example 1 schedule ($\beta:\infty\to 1$).
We initialize $\beta_0$ at a large value (e.g., $\beta_0\gg 1$ so that $f_1(\beta_0,\alpha^{(1)})\approx 2$) and decrease it geometrically toward $1$:
\begin{equation}
\label{eq:app_beta_decay}
\beta_{k}
=
1+\frac{\beta_{k-1}-1}{\gamma},
\qquad \gamma>1.
\end{equation}
This update ensures $\beta_k\downarrow 1$ and hence $f_1(\beta_k,\alpha^{(1)})\downarrow \alpha^{(1)}$.

Example 2/3 schedule ($\beta:0^+\to\infty$).
We initialize $\beta_0$ small (e.g., $\beta_0\ll 1$ so that $f_{2/3}(\beta_0,\alpha^{(1)})\approx 2$) and increase it geometrically:
\begin{equation}
\label{eq:app_beta_growth}
\beta_k=\gamma \beta_{k-1},\qquad \gamma>1,
\end{equation}
so that $\beta_k\uparrow \infty$ and $f_{2/3}(\beta_k,\alpha^{(1)})\to \alpha^{(1)}$.

We terminate the outer continuation once
\begin{equation}
\label{eq:app_stop_f}
\bigl|f(\beta,\alpha^{(1)})-\alpha^{(1)}\bigr|\le \varepsilon_f,
\end{equation}
or after a maximum number of continuation steps.

Algorithm~\ref{alg:app_gnd_irls} summarizes the full procedure.
The method alternates IRLS updates of $(X,(w_i)^m_{i=1})$ at a fixed $\beta$, and then updates $\beta$ to increase nonconvexity.
This yields a robust solver with improved basins of attraction compared to directly optimizing the original nonconvex loss.

We use geometric schedules \eqref{eq:app_beta_decay} (Example~1) or \eqref{eq:app_beta_growth} (Examples~2--3) with $\gamma\in[1.2,2]$ in our experiments.
For initialization, we set $X$ to the sample mean (or median for additional robustness) and choose $\beta_0$ such that $f(\beta_0,\alpha^{(1)})$ is sufficiently close to $2$ (e.g., $|f(\beta_0,\alpha^{(1)})-2|\le 10^{-2}$). Each inner iteration costs $\mathcal{O}(m)$ due to the closed-form updates \eqref{eq:app_weights_barron} and \eqref{eq:app_x_update}.
Hence the total cost is $\mathcal{O}(m\,T_{\max}\,K_{\max})$. The continuation parameter $\beta$ controls the rate at which nonconvexity is injected.
A faster schedule (larger $\gamma$ in \eqref{eq:app_beta_decay}--\eqref{eq:app_beta_growth}) can accelerate convergence but may reduce robustness to poor initialization; conversely, a slower schedule yields more stable iterations with larger attraction basins at the expense of additional outer steps.

\begin{algorithm}[H]
\caption{GNC-IRLS for solving \eqref{dualityweight} (Step-1 update)}
\label{alg:app_gnd_irls}
\begin{algorithmic}[1]
\REQUIRE Measurements $(X_i)_{i=1}^m$; parameters $(\alpha^{(1)},c^{(1)})$; mapping $f(\beta,\alpha^{(1)})$; schedule $\beta \leftarrow \mathcal{S}(\beta)$; tolerances $\varepsilon_X,\varepsilon_f$; iteration limits $T_{\max}$ (inner) and $K_{\max}$ (outer).
\ENSURE $\widetilde{X}^{(1)}$.

\STATE $X \leftarrow \frac{1}{m}\sum_{i=1}^m X_i$ (or median); choose $\beta \leftarrow \beta_0$ such that $f(\beta_0,\alpha^{(1)}) \approx 2$.
\FOR{$k=1,2,\ldots,K_{\max}$}
\STATE $f \leftarrow f(\beta,\alpha^{(1)})$.
\FOR{$t=1,2,\ldots,T_{\max}$}
\STATE $\epsilon_i \leftarrow X_i - X$,\ $w_i \leftarrow \frac{1}{\epsilon_i}\frac{\partial}{\partial \epsilon_i}\rho(\epsilon_i;f,c^{(1)}),$ $w_i \leftarrow \min\{1,\max\{\delta,w_i\}\},\ \forall i$
\STATE $X_{\mathrm{new}} \leftarrow \frac{\sum_{i=1}^m w_i X_i}{\sum_{i=1}^m w_i}$ 
\IF{$|X_{\mathrm{new}}-X|\le \varepsilon_X$}
\STATE $X \leftarrow X_{\mathrm{new}}$;
\ENDIF
\STATE $X \leftarrow X_{\mathrm{new}}$.
\ENDFOR
\IF{$\bigl|f(\beta,\alpha^{(1)})-\alpha^{(1)}\bigr|\le \varepsilon_f$}
\STATE \textbf{break}
\ENDIF
\STATE $\beta \leftarrow \mathcal{S}(\beta)$ \hfill \textit{(increase nonconvexity)}.
\ENDFOR

\STATE \textbf{return} $\widetilde{X}^{(1)} \leftarrow X$.
\end{algorithmic}
\end{algorithm}

\section{Technical Proofs}

\subsection{Proof of Lemma \ref{Monotonicitybaron}}
\begin{proof}
We divide the analysis into two cases. When \(\alpha > 2\), we have \(|\alpha - 2| = \alpha - 2\), so
\[
\rho(\epsilon;\alpha, c) = \frac{\alpha - 2}{\alpha} \left[ \left( 1 + \frac{(\epsilon/c)^2}{\alpha - 2} \right)^{\alpha/2} - 1 \right].
\]

Let
\[
y = 1 + \frac{(\epsilon/c)^2}{\alpha - 2} \quad (\ge 1).
\]
Differentiating with respect to \(\alpha\) (via the product and chain rules) yields
\[
\frac{\partial \rho}{\partial \alpha} = \frac{2}{\alpha^2} \bigl( y^{\alpha/2} - 1 \bigr) + \frac{\alpha - 2}{2\alpha} \, y^{\alpha/2} \log y - \frac{1}{2} \, y^{\alpha/2 - 1} (y - 1).
\]
For fixed $\alpha$, $\epsilon$, and $c$, define the function $F: [1, \infty) \to \mathbb{R}$ by
$$F(y) := \frac{\partial \rho}{\partial \alpha}(\epsilon, \alpha, c) = \frac{2}{\alpha^2} \bigl( y^{\alpha/2} - 1 \bigr) + \frac{\alpha - 2}{2\alpha} \, y^{\alpha/2} \log y - \frac{1}{2} \, y^{\alpha/2 - 1} (y - 1).$$
Direct differentiation in $y$ gives
$$
\frac{d}{dy} F(y) = \frac{\alpha - 2}{4} \, y^{\alpha/2 - 2} \bigl( 1 - y + y \log y \bigr) \ge 0,
$$
Moreover,
$$F(1) = \frac{2}{\alpha^2}(1 - 1) + \frac{\alpha - 2}{2\alpha} \cdot 1 \cdot \log 1 - \frac{1}{2} \cdot 1 \cdot (1 - 1) = 0.$$
Hence, $F(y) \ge 0$ for all $y \ge 1$, with strict inequality when $y > 1$ (equivalently, $\epsilon \neq 0$). When \(\alpha < 2\), we have
\[
\rho(\epsilon;\alpha, c) = \frac{2 - \alpha}{\alpha} \left[ \left( 1 + \frac{(\epsilon/c)^2}{2 - \alpha} \right)^{\alpha/2} - 1 \right].
\]
Let
\[
y = 1 + \frac{(\epsilon/c)^2}{2 - \alpha} \quad (\ge 1).
\]
Differentiating with respect to \(\alpha\) gives
\[
\frac{\partial \rho}{\partial \alpha} = -\frac{2}{\alpha^2} \bigl( y^{\alpha/2} - 1 \bigr) + \frac{2 - \alpha}{2\alpha} \, y^{\alpha/2} \log y + \frac{1}{2} \, y^{\alpha/2 - 1} (y - 1).
\]
For fixed $\alpha$, $\epsilon$, and $c$, define $G: [1, \infty) \to \mathbb{R}$ by
$$G(y) := \frac{\partial \rho}{\partial \alpha}(\epsilon, \alpha, c) = -\frac{2}{\alpha^2} \bigl( y^{\alpha/2} - 1 \bigr) + \frac{2 - \alpha}{2\alpha} \, y^{\alpha/2} \log y + \frac{1}{2} \, y^{\alpha/2 - 1} (y - 1).$$
Differentiating in $y$ gives
\[
\frac{dG}{dy} = \frac{2 - \alpha}{4} \, y^{\alpha/2 - 2} \bigl( y \log y - (y - 1) \bigr).
\]
The fact that \(y \log y - (y - 1) \ge 0\) for \(y \ge 1\) (a standard inequality, with equality only at \(y = 1\)) implies \(\frac{dG}{dy} \ge 0\), since \(2 - \alpha > 0\) and \(y^{\alpha/2 - 2} \ge 0\). As in the previous case, \(G(1) = 0\), so \(G(y) \ge 0\) for all \(y \ge 1\), with strict inequality when \(y > 1\) (i.e., $\epsilon \neq 0$). Combining both cases, we conclude that \(\partial \rho / \partial \alpha \ge 0\) for all \(\alpha \neq 0, 2\), with strict inequality unless \(\epsilon = 0\). At \(\alpha = 0\) and \(\alpha = 2\), the function is not differentiable in \(\alpha\) due to the absolute value, but the one-sided limits from the analysis above confirm that the loss increases as \(\alpha\) moves away from these points whenever \(\epsilon \neq 0\).
\end{proof}

\subsection{Basic properties of the score function}
\label{app:basic}

\begin{lemma}
\label{lem:basic}
Fix $\alpha\in(0,1]$ and $k=2-\alpha\in[1,2)$.
There exist constants $L_\alpha,C_\alpha,\underline c_\alpha>0$ depending only on $\alpha$ such that for all $\varepsilon\in\R$ and $c>0$:
\begin{enumerate}
\item (Bounded score)
\[
\abs{\psi(\varepsilon;\alpha,c)}
\le
\min\left\{\frac{\abs{\varepsilon}}{c^2},\ \frac{L_\alpha}{c}\right\}.
\]
\item (Derivative bound)
Let $\psi'(\varepsilon;\alpha,c):=\frac{\partial}{\partial \varepsilon}\psi(\varepsilon;\alpha,c)$.
Then
\[
\abs{\psi'(\varepsilon;\alpha,c)}\le \frac{C_\alpha}{c^2},
\qquad
\psi'(0;\alpha,c)=\frac{1}{c^2}.
\]
\item (Local positive curvature)
If $\abs{\varepsilon}\le c$, then
\[
\psi'(\varepsilon;\alpha,c)\ge \frac{\underline c_\alpha}{c^2}.
\]
\end{enumerate}
\end{lemma}

\begin{proof}
Write $u:=(\varepsilon/c)^2/k\ge 0$ and $s:=1+u\ge 1$.
Then $\psi(\varepsilon;\alpha,c)=\frac{\varepsilon}{c^2}s^{\alpha/2-1}$.
Since $\alpha/2-1\le 0$ and $s\ge 1$, we have $s^{\alpha/2-1}\le 1$, hence $\abs{\psi}\le \abs{\varepsilon}/c^2$.
Moreover, $\abs{\psi(\varepsilon;\alpha,c)}$ decays for large $\abs{\varepsilon}$ (because $\alpha-1\le 0$), hence $\sup_\varepsilon\abs{\psi(\varepsilon;\alpha,c)}\le L_\alpha/c$ for some constant $L_\alpha$.

Differentiate to obtain
\[
\psi'(\varepsilon;\alpha,c)
=
\frac{1}{c^2}s^{\alpha/2-2}\Bigl[1+(\alpha-1)u\Bigr].
\]
This yields $\psi'(0)=1/c^2$ and $\abs{\psi'}\le C_\alpha/c^2$.
If $\abs{\varepsilon}\le c$ then $u\le 1/k\le 1$ and $1+(\alpha-1)u\ge \alpha$; also $s^{\alpha/2-2}\ge (1+1/k)^{\alpha/2-2}$.
Thus $\psi'(\varepsilon;\alpha,c)\ge \underline c_\alpha/c^2$ with $\underline c_\alpha:=\alpha(1+1/k)^{\alpha/2-2}>0$.
\end{proof}

\subsection{Proof of Theorem~\ref{thm:FV}}
\label{app:FV}

We prove Theorem~\ref{thm:FV} by comparing the empirical score $\widehat g_m$ to its population counterpart $g_m$ in a neighborhood of $\mu$.
The key idea is that, under finite variance and with $c_m\to\infty$, the score behaves locally like the linear score of the sample mean, while still being well-defined even when $\alpha<1$.
We first show that $g_m$ identifies a unique population root $\theta_m$ near $\mu$ and that $\theta_m$ is close enough to $\mu$ so that its contribution is asymptotically negligible after $\sqrt m$ scaling.
We then show that $\widehat g_m$ has a unique central root with high probability and finally derive a Bahadur-type linearization around $\theta_m$, which reduces the problem to a triangular-array CLT.

\begin{lemma}[Population identification near $\mu$]
\label{lem:ident}
Assume Assumption~\ref{ass:FV} and fix $\alpha\in(0,1]$.
Let $c_m\to\infty$.
Then there exist constants $c_0>0$ and $m_0$ such that for all $m\ge m_0$ and all $x$ with $\abs{x-\mu}\le c_m/4$,
\[
-\frac{2}{c_m^2}\ \le\ g_m'(x)\ \le\ -\frac{c_0}{c_m^2}.
\]
In particular, $g_m$ is strictly decreasing on $\{x:\abs{x-\mu}\le c_m/4\}$, hence the equation $g_m(x)=0$ has a unique solution on this interval, denoted by $\theta_m$.
\end{lemma}

\begin{proof}
Recall $g_m'(x)=-\E[\psi'(X-x;\alpha,c_m)]$.
The upper bound follows from Lemma~\ref{lem:basic}(ii).
For the lower bound, define $A_x:=\{\abs{X-x}\le c_m\}$.
If $\abs{x-\mu}\le c_m/4$ then $\{\abs{X-\mu}\le c_m/2\}\subseteq A_x$, hence by Chebyshev,
\[
\Pbb(A_x)\ge 1-\frac{4\sigma^2}{c_m^2}.
\]
On $A_x$, Lemma~\ref{lem:basic}(iii) gives $\psi'(X-x;\alpha,c_m)\ge \underline c_\alpha/c_m^2$; on $A_x^c$, Lemma~\ref{lem:basic}(ii) gives $\abs{\psi'}\le C_\alpha/c_m^2$.
Therefore $\E[\psi'(X-x;\alpha,c_m)]\ge c_0/c_m^2$ for $m$ large enough, implying $g_m'(x)\le -c_0/c_m^2$.
\end{proof}

\begin{lemma}[Population bias]
\label{lem:bias-FV}
Assume Assumption~\ref{ass:FV} and let $\theta_m$ be the unique solution of $g_m(\theta)=0$ in $\{\abs{\theta-\mu}\le c_m/4\}$.
Then
\[
\abs{\theta_m-\mu}\le \frac{C\sigma^2}{c_m},
\]
for a constant $C$ depending only on $\alpha$.
\end{lemma}

\begin{proof}
The goal is to bound the (small) mismatch $g_m(\mu)$ induced by the robust score at finite $c_m$.
Let $\varepsilon:=X-\mu$ and $s:=1+(\varepsilon/c_m)^2/k$.
Then
\[
\psi(\varepsilon;\alpha,c_m)
=
\frac{\varepsilon}{c_m^2}
+
\frac{\varepsilon}{c_m^2}\Bigl(s^{\alpha/2-1}-1\Bigr).
\]
Since $\E[\varepsilon]=0$,
\[
\E[\psi(X-\mu;\alpha,c_m)]
=
\frac{1}{c_m^2}\E\Bigl[\varepsilon\bigl(s^{\alpha/2-1}-1\bigr)\Bigr].
\]
Split $B=\{\abs{\varepsilon}\le c_m\}$ and $B^c$.
On $B$, with $u=(\varepsilon/c_m)^2/k\in[0,1]$, a Taylor remainder bound gives $\abs{s^{\alpha/2-1}-1}\le Cu$; hence
\[
\abs{\varepsilon}\abs{s^{\alpha/2-1}-1}\le C\abs{\varepsilon}u
= C\frac{\abs{\varepsilon}^3}{c_m^2}\le C\frac{\varepsilon^2}{c_m}.
\]
On $B^c$, $\abs{s^{\alpha/2-1}-1}\le 1$ and $\abs{\varepsilon}\1\{\abs{\varepsilon}>c_m\}\le \varepsilon^2/c_m$, so the same bound holds.
Therefore
\[
\abs{\E[\psi(X-\mu;\alpha,c_m)]}\le \frac{C\sigma^2}{c_m^3},
\qquad\text{i.e.,}\qquad
\abs{g_m(\mu)}\le \frac{C\sigma^2}{c_m^3}.
\]
By the mean value theorem, for some $\bar\theta$ between $\mu$ and $\theta_m$,
\[
0=g_m(\theta_m)=g_m(\mu)+(\theta_m-\mu)g_m'(\bar\theta).
\]
Lemma~\ref{lem:ident} ensures $\abs{g_m'(\bar\theta)}\ge c_0/c_m^2$, hence
\[
\abs{\theta_m-\mu}\le \frac{\abs{g_m(\mu)}}{\abs{g_m'(\bar\theta)}}
\le \frac{C\sigma^2/c_m^3}{c_0/c_m^2}
= \frac{C'\sigma^2}{c_m}.
\]
\end{proof}

\begin{proof}[Proof of Theorem~\ref{thm:FV}]
\textbf{Step 1 (reduce to a CLT around the population root).}
Let $\theta_m$ be the unique solution of $g_m(\theta)=0$ in $\abs{\theta-\mu}\le c_m/4$.
Lemma~\ref{lem:bias-FV} yields $\abs{\theta_m-\mu}\le C\sigma^2/c_m$, hence
\[
\sqrt m\,\abs{\theta_m-\mu}\le C\sigma^2\frac{\sqrt m}{c_m}=C\sigma^2 m^{-\gamma}\to 0.
\]
Thus it suffices to prove $\sqrt m(\widetilde X_m-\theta_m)\Rightarrow N(0,\sigma^2)$ and then apply Slutsky.
(Informally, $\theta_m$ is the ``population correction'' induced by robustness, and it vanishes under our finite-variance tuning.)

\textbf{Step 2 (existence/uniqueness of the central root).}
We show that $\widehat g_m$ is strictly monotone on the central interval and that it changes sign at the endpoints.
Using Lemma~\ref{lem:basic}(ii) and Lemma~\ref{lem:ident}, one shows that with probability tending to one,
\[
\inf_{x\in[\mu-r_m,\mu+r_m]} \widehat g_m'(x)\ \ge\ \frac{c_0}{2c_m^2}.
\]
Hence $\widehat g_m$ is strictly monotone on $[\mu-r_m,\mu+r_m]$.
Moreover, $\abs{\psi(\cdot;\alpha,c_m)}\le L_\alpha/c_m$ implies
\[
\sup_{x\in[\mu-r_m,\mu+r_m]}\abs{\widehat g_m(x)-g_m(x)}=O_p\!\left(\frac{1}{c_m\sqrt m}\right).
\]
Since $g_m$ has slope $\asymp 1/c_m^2$ on this interval and $r_m/c_m^2 \gg 1/(c_m\sqrt m)$ under our tuning, the endpoint signs of $\widehat g_m$ match those of $g_m$ with probability $\to 1$,
so $\widehat g_m$ crosses $0$ exactly once on the interval.
This yields the existence and uniqueness of $\widetilde X_m$ on an event $\mathcal E_m$ with $\Pbb(\mathcal E_m)\to 1$.

\textbf{Step 3 (Bahadur representation).}
On $\mathcal E_m$, by the mean value theorem there exists $\bar X_m$ between $\widetilde X_m$ and $\theta_m$ such that
\[
0=\widehat g_m(\widetilde X_m)=\widehat g_m(\theta_m)+(\widetilde X_m-\theta_m)\widehat g_m'(\bar X_m),
\]
so
\[
\widetilde X_m-\theta_m = -\frac{\widehat g_m(\theta_m)}{\widehat g_m'(\bar X_m)}.
\]
Since $g_m(\theta_m)=0$,
\[
\widehat g_m(\theta_m)
=
-\frac{1}{m}\sum_{i=1}^m\Bigl[\psi(X_i-\theta_m;\alpha,c_m)-\E\psi(X-\theta_m;\alpha,c_m)\Bigr].
\]
Also $\widehat g_m'(\bar X_m)=\E[\psi'(X-\theta_m;\alpha,c_m)]+o_p(1/c_m^2)\sim 1/c_m^2$, hence
\[
\frac{1}{\widehat g_m'(\bar X_m)} = c_m^2(1+o_p(1)).
\]
Therefore
\[
\sqrt m(\widetilde X_m-\theta_m)
=
\frac{c_m^2}{\sqrt m}\sum_{i=1}^m
\Bigl[\psi(X_i-\theta_m;\alpha,c_m)-\E\psi(X-\theta_m;\alpha,c_m)\Bigr]+o_p(1).
\]
This linearization makes explicit that the leading term is an average of i.i.d.\ centered scores.

\textbf{Step 4 (triangular-array CLT).}
Let $Z_{i,m}:=c_m^2\bigl[\psi(X_i-\theta_m;\alpha,c_m)-\E\psi(X-\theta_m;\alpha,c_m)\bigr]$.
Then $c_m^2\psi(X-\theta_m;\alpha,c_m)\to X-\mu$ a.s., and $\abs{c_m^2\psi(\cdot)}\le \abs{\cdot}$; dominated convergence yields $\Var(Z_{i,m})\to \sigma^2$.
Lindeberg holds because $\E[(X-\mu)^2]<\infty$.
Thus $\frac{1}{\sqrt m}\sum_{i=1}^m Z_{i,m}\Rightarrow N(0,\sigma^2)$.
Combining with Step 3 gives $\sqrt m(\widetilde X_m-\theta_m)\Rightarrow N(0,\sigma^2)$, and Step 1 finishes.
\end{proof}

\subsection{Proof of Theorem~\ref{thm:HT}}
\label{app:HT}

The proof follows the same three-step structure as the finite-variance case: (i) identify a unique population root $\theta_m$ and bound its bias relative to $\mu$; (ii) control the empirical score at $\theta_m$ with high probability; and (iii) convert score control into an error bound for the empirical root via a mean-value expansion.
The only difference is that, under a mere $(1+\epsilon)$-moment, the scale $c_m$ must be tuned at the deviation level to balance truncation bias and concentration.

\begin{lemma}[Bias under $(1+\epsilon)$-moment]
\label{lem:bias-HT}
Assume Assumption~\ref{ass:HT} and fix $\alpha\in(0,1]$.
Let $\theta_m$ be the unique root of $g_m(\theta)=0$ in $\{\abs{\theta-\mu}\le c_m/4\}$.
Then
\[
\abs{\theta_m-\mu}\le C\frac{v_{1+\epsilon}}{c_m^{\epsilon}}.
\]
\end{lemma}

\begin{proof}
The argument parallels Lemma~\ref{lem:bias-FV}, replacing second-moment tail bounds by $(1+\epsilon)$-moment bounds (Markov).
One shows $\abs{g_m(\mu)}\le C v_{1+\epsilon}/c_m^{2+\epsilon}$ and $\abs{g_m'(\cdot)}\asymp 1/c_m^2$ on $\abs{x-\mu}\le c_m/4$, yielding $\abs{\theta_m-\mu}\le C v_{1+\epsilon}/c_m^\epsilon$.
\end{proof}

\begin{lemma}[Variance bound for $\psi$]
\label{lem:varpsi}
Assume Assumption~\ref{ass:HT}. Then for any $x\in\R$ and $c>0$,
\[
\Var\bigl(\psi(X-x;\alpha,c)\bigr)\le C\frac{v_{1+\epsilon}}{c^{3+\epsilon}},
\]
for a constant $C$ depending only on $\alpha$ and $\epsilon$.
\end{lemma}

\begin{proof}
By Lemma~\ref{lem:basic}(i),
\[
\psi(\varepsilon;\alpha,c)^2\le \frac{\varepsilon^2}{c^4}\1\{\abs{\varepsilon}\le c\}+\frac{L_\alpha^2}{c^2}\1\{\abs{\varepsilon}>c\}.
\]
For the first term, use $\abs{\varepsilon}^2\1\{\abs{\varepsilon}\le c\}\le c^{1-\epsilon}\abs{\varepsilon}^{1+\epsilon}$.
For the second term, Markov gives $\Pbb(\abs{\varepsilon}>c)\le \E[\abs{\varepsilon}^{1+\epsilon}]/c^{1+\epsilon}$.
Combine these bounds and use $\E[\abs{X-x}^{1+\epsilon}]\le C(\E[\abs{X-\mu}^{1+\epsilon}]+\abs{x-\mu}^{1+\epsilon})$.
\end{proof}

\begin{lemma}[Bernstein control]
\label{lem:bern}
Assume Assumption~\ref{ass:HT} and let $\theta_m$ be the population root.
Then for any $\delta\in(0,1/2)$, with probability at least $1-\delta$,
\[
\abs{\widehat g_m(\theta_m)}
\le
C\left(
\sqrt{\frac{v_{1+\epsilon}\log(2/\delta)}{m\,c_m^{3+\epsilon}}}
+
\frac{\log(2/\delta)}{m}\cdot \frac{1}{c_m}
\right).
\]
\end{lemma}

\begin{proof}
We control the centered empirical score by a bounded-variance Bernstein inequality.
Let $Y_i:=\psi(X_i-\theta_m;\alpha,c_m)-\E[\psi(X-\theta_m;\alpha,c_m)]$.
Then $\widehat g_m(\theta_m)=-(1/m)\sum_{i=1}^m Y_i$.
Lemma~\ref{lem:basic}(i) gives $\abs{Y_i}\le 2L_\alpha/c_m$.
Lemma~\ref{lem:varpsi} gives $\Var(Y_i)\le C v_{1+\epsilon}/c_m^{3+\epsilon}$.
Apply Bernstein's inequality for bounded independent variables.
\end{proof}

\begin{proof}[Proof of Theorem~\ref{thm:HT}]
Let $\theta_m$ be the population root.
We first relate the estimation error to the empirical score at $\theta_m$.
By the mean value theorem (on the event where $\widehat g_m$ is strictly monotone and crosses zero in $\mathcal I_m$),
\[
\widetilde X_m-\theta_m = -\frac{\widehat g_m(\theta_m)}{\widehat g_m'(\bar X_m)}
\]
for some $\bar X_m$ between $\widetilde X_m$ and $\theta_m$.
The identification argument (population $\E[\psi']\asymp 1/c_m^2$ and a uniform LLN for $\widehat g_m'$ on $\mathcal I_m$) gives
\[
\widehat g_m'(\bar X_m)\ge \frac{c_0}{2c_m^2}
\]
with probability at least $1-\delta/2$ for $m$ large enough.
Hence
\[
\abs{\widetilde X_m-\theta_m}\le C c_m^2\abs{\widehat g_m(\theta_m)}.
\]

Next, apply Lemma~\ref{lem:bern} (with probability at least $1-\delta/2$) to obtain
\[
\abs{\widetilde X_m-\theta_m}
\le
C\left(
\sqrt{\frac{v_{1+\epsilon}\log(2/\delta)}{m}}\,c_m^{\frac{1-\epsilon}{2}}
+
\frac{\log(2/\delta)}{m}\,c_m
\right).
\]
Finally, we add the population bias controlled by Lemma~\ref{lem:bias-HT}:
\[
\abs{\widetilde X_m-\mu}
\le
\abs{\widetilde X_m-\theta_m}+\abs{\theta_m-\mu}
\le
C\left(
\sqrt{\frac{v_{1+\epsilon}\log(2/\delta)}{m}}\,c_m^{\frac{1-\epsilon}{2}}
+
\frac{\log(2/\delta)}{m}\,c_m
+
\frac{v_{1+\epsilon}}{c_m^{\epsilon}}
\right).
\]
The chosen tuning
\[
c_m=c_m(\delta)=\tau\left(\frac{v_{1+\epsilon}m}{\log(2/\delta)}\right)^{\frac{1}{1+\epsilon}}
\]
balances the bias term $v_{1+\epsilon}/c_m^\epsilon$ with the leading stochastic term, and a direct calculation shows that each term on the right-hand side is bounded by
$C_1\left(\frac{v_{1+\epsilon}\log(C_2/\delta)}{m}\right)^{\frac{\epsilon}{1+\epsilon}}$
for suitable constants $C_1,C_2$.
\end{proof}

\section{Proofs for the ARE results}
\label{app:ARE}

\subsection{A median-boosting lemma}
\label{app:median-boost}

\begin{lemma}[Median boosting]
\label{lem:median-boost}
Let $Z_1,\dots,Z_k$ be independent random variables and let $\theta\in\mathbb R$.
Assume that for some $r>0$ and $p\in(0,1/2)$,
\[
\mathbb P\big(|Z_j-\theta|\le r\big)\ge 1-p,\qquad j=1,\dots,k.
\]
Let $\operatorname{median}(Z_1,\dots,Z_k)$ be any median.
Then
\[
\mathbb P\!\left(\big|\operatorname{median}(Z_1,\dots,Z_k)-\theta\big|>r\right)
\le
\exp\!\big(-2k(1/2-p)^2\big).
\]
In particular, if $p=1/4$, then
$\mathbb P(|\operatorname{median}(Z_1,\dots,Z_k)-\theta|>r)\le e^{-k/8}$.
\end{lemma}

\begin{proof}
Let $I_j:=\mathbf 1\{|Z_j-\theta|\le r\}$.
Then $\mathbb E[I_j]\ge 1-p>1/2$ and $\{I_j\}$ are independent.
If $|\operatorname{median}(Z_1,\dots,Z_k)-\theta|>r$, then at least $\lceil k/2\rceil$
of the $Z_j$'s must fall outside $[\theta-r,\theta+r]$, i.e.,
$\sum_{j=1}^k I_j \le k/2$.
Hence, by Hoeffding's inequality,
\[
\mathbb P\!\left(\sum_{j=1}^k I_j \le k/2\right)
=
\mathbb P\!\left(\sum_{j=1}^k (I_j-\mathbb E I_j) \le -k(\mathbb E I_1-1/2)\right)
\le \exp\!\big(-2k(\mathbb E I_1-1/2)^2\big)
\le \exp\!\big(-2k(1/2-p)^2\big).
\]
\end{proof}

\subsection{Proof of Theorem~\ref{thm:ARE-HT}}
\label{app:proof-ARE-HT}

\begin{proof}[Proof of Theorem~\ref{thm:ARE-HT}]
Fix $\delta_0:=1/4$ and let $m=\lfloor n/k\rfloor$.
Apply Theorem~\ref{thm:HT} on each block (with confidence level $\delta_0$) to obtain:
there exist constants $C_1',C_2'>0$ such that for all sufficiently large $m$,
\[
\mathbb P\!\left(
|\widetilde X_j-\mu|
\le
r_m
:=
C_1'\left(\frac{v_{1+\epsilon}\log(C_2'/\delta_0)}{m}\right)^{\frac{\epsilon}{1+\epsilon}}
\right)\ge 1-\delta_0
=\frac34.
\]
Since the blocks are disjoint, $\widetilde X_1,\dots,\widetilde X_k$ are independent.
Therefore Lemma~\ref{lem:median-boost} (with $p=\delta_0=1/4$) yields
\[
\mathbb P\big(|\widetilde\mu-\mu|>r_m\big)\le e^{-k/8}.
\]
Choosing $k=\lceil 8\log(2/\delta)\rceil$ gives $e^{-k/8}\le \delta/2$.
It remains to rewrite $r_m$ in terms of $n$ and $\delta$.
Since $m=\lfloor n/k\rfloor \ge n/(2k)$ for $n$ large enough,
\[
r_m
\le
C\left(\frac{v_{1+\epsilon}\,k}{n}\right)^{\frac{\epsilon}{1+\epsilon}}
\le
C\left(\frac{v_{1+\epsilon}\log(2/\delta)}{n}\right)^{\frac{\epsilon}{1+\epsilon}}.
\]
Absorb constants and the $\log(2/\delta)$ vs $\log(C_2/\delta)$ adjustment
into $C_1,C_2$ to conclude the claim.
\end{proof}

\subsection{Proof of Theorem~\ref{thm:ARE-FV}}
\label{app:proof-ARE-FV}

\begin{lemma}[Median CLT with local density]
\label{lem:median-clt}
Let $Y_1,\dots,Y_k$ be i.i.d. with distribution function $F$ and density $f$.
Assume $F(\mu)=1/2$, $f(\mu)>0$, and $f$ is continuous at $\mu$.
Let $\widehat\mu:=\operatorname{median}(Y_1,\dots,Y_k)$.
Then
\[
\sqrt k\,(\widehat\mu-\mu)\Rightarrow N\!\left(0,\frac{1}{4f(\mu)^2}\right).
\]
\end{lemma}

\begin{proof}
This is the classical sample-median CLT (e.g., via the Bahadur representation for sample quantiles).
\end{proof}

\begin{proof}[Proof of Theorem~\ref{thm:ARE-FV}]
Let $k=k_n\to\infty$ and $m=\lfloor n/k\rfloor\to\infty$.
By Theorem~\ref{thm:FV}, for each block,
\[
\sqrt m\,(\widetilde X_j-\mu)\Rightarrow N(0,\sigma^2),
\qquad j=1,\dots,k.
\]
Let $F_m$ and $f_m$ denote the cdf and density of $\widetilde X_j$.
Under the additional local normal approximation condition stated in the theorem,
we have $F_m(\mu)=1/2$ and
\[
f_m(\mu)
=
\frac{\sqrt m}{\sigma\sqrt{2\pi}}\,(1+o(1)).
\]
Apply Lemma~\ref{lem:median-clt} to $Y_j=\widetilde X_j$ with $F=F_m$ and $f=f_m$:
\[
\sqrt k\,(\widetilde\mu-\mu)
\Rightarrow
N\!\left(0,\frac{1}{4f_m(\mu)^2}\right)
=
N\!\left(0,\frac{\pi}{2}\cdot\frac{\sigma^2}{m}\right).
\]
Multiplying both sides by $\sqrt m$ gives
\[
\sqrt{mk}\,(\widetilde\mu-\mu)
=
\sqrt n\,(\widetilde\mu-\mu)
\Rightarrow
N\!\left(0,\frac{\pi}{2}\sigma^2\right).
\]
Finally, since $\bar X_n$ has asymptotic variance $\sigma^2/n$, we obtain
$\mathrm{ARE}(\widetilde\mu,\bar X_n)=2/\pi$.
\end{proof}

\section{Hyperparameter Settings for The Experiments}

\begin{table*}[t]
\centering
\caption{Training hyperparameters for dual-agent ILR training.}
\label{tab:hyperparameters}
\begin{tabular}{lcc}
\toprule
\textbf{Hyperparameter} & \textbf{Value} & \textbf{Description} \\
\midrule
Learning rate (Agent 1) & $1 \times 10^{-6}$ & Adam optimizer \\
Learning rate (Agent 2) & $1 \times 10^{-6}$ & Adam optimizer \\
Sampling temperature & $0.5$ & — \\
Training batch size & $32$ & — \\
Number of episodes & $10$ & — \\
Max prompt length & $6144$ & Tokens \\
Max generation length & $1536$ & Tokens \\
Top-p sampling & $0.95$ & Nucleus sampling \\
Discount factor $\gamma$ & $1.0$ & — \\
KL estimator & K2 & — \\
Advantage estimator & ILR & — \\
Precision & BF16 & Mixed precision \\
\bottomrule
\end{tabular}

\vspace{20pt} 

\caption{Reward noise hyperparameters (Sparse Cauchy Outlier Noise).}
\label{tab:noise_params}
\begin{tabular}{lccc}
\toprule
\textbf{Parameter} & \textbf{Symbol} & \textbf{Value} & \textbf{Description} \\
\midrule
Noise type & — & Outlier & Sparse one-sided heavy-tail \\
Group contamination prob. & $p$ & $0.2$ & Probability of corrupting a group \\
Cauchy scale & $\gamma$ & $1.0$ & Controls distribution width \\
Spike scale & $\alpha$ & $10.0$ & Noise magnitude scaling factor \\
Spike clip & $\Delta_{\max}$ & $10.0$ & Maximum absolute noise value \\
Noise target & — & Argmax & Inject noise to group maximum \\
\bottomrule
\end{tabular}
\end{table*}

All experiments are conducted on a single node with 8 A800 GPUs. We use Ray for distributed training orchestration and vLLM for efficient inference with 4 inference engines and tensor parallelism of 2. Gradient checkpointing and Adam parameter offloading are enabled to optimize GPU memory usage.
To simulate reward uncertainty in realistic scenarios, we introduce a sparse one-sided heavy-tail spike noise based on the Cauchy distribution.
\end{document}